\documentclass[mnsc,non-blindrev]{informs3} 
\OneAndAHalfSpacedXI 

\usepackage{endnotes}
\usepackage[ruled]{algorithm}
\usepackage{algorithmicx,algpseudocode}
\usepackage{graphicx,hyperref}
\usepackage{subfigure,caption,microtype,booktabs} 
\usepackage{amsfonts,array}

\usepackage{amsthm,amsmath,amssymb,soul}

\usepackage{bbm}
\usepackage{bm,comment}
\usepackage{color}
\usepackage{dirtytalk}
\usepackage{dsfont}
\usepackage{enumerate}
\usepackage{graphicx}
\usepackage{listings}
\usepackage{mathtools}
\usepackage{natbib}
\usepackage{subfigure}
\usepackage{times,xspace,enumitem}
\usepackage[usenames,dvipsnames]{xcolor}
\setlist[itemize]{leftmargin=*}
\usepackage[capitalize,noabbrev]{cleveref}
\usepackage[backgroundcolor=white]{todonotes}

\newcommand{\commentout}[1]{}
\newcommand{\junk}[1]{}

\Crefname{corollary}{Corollary}{Corollaries}
\Crefname{proposition}{Proposition}{Propositions}
\Crefname{theorem}{Theorem}{Theorems}
\Crefname{definition}{Definition}{Definitions}
\Crefname{assumption}{Assumption}{Assumptions}
\Crefname{example}{Example}{Examples}
\Crefname{remark}{Remark}{Remarks}
\Crefname{setting}{Setting}{Settings}
\Crefname{lemma}{Lemma}{Lemmas}

\usepackage{thmtools}
\declaretheorem[name=Theorem,refname={Theorem,Theorems},Refname={Theorem,Theorems}]{theorem}
\declaretheorem[name=Lemma,refname={Lemma,Lemmas},Refname={Lemma,Lemmas},sibling=theorem]{lemma}

\declaretheorem[name=Proposition,refname={Proposition,Propositions},Refname={Proposition,Propositions},sibling=theorem]{proposition}

\declaretheorem[name=Example,refname={Example,Examples},Refname={Example,Examples}]{example}
\declaretheorem[name=Remark,refname={Remark,Remarks},Refname={Remark,Remarks}]{remark}

\newcommand{\cA}{\mathcal{A}}

\newcommand{\cC}{\mathcal{C}}
\newcommand{\cD}{\mathcal{D}}
\newcommand{\cE}{\mathcal{E}}

\newcommand{\cH}{\mathcal{H}}

\newcommand{\cN}{\mathcal{N}}

\newcommand{\cQ}{\mathcal{Q}}

\newcommand{\cS}{\mathcal{S}}

\newcommand{\cX}{\mathcal{X}}

\newcommand{\eps}{\varepsilon}

\newcommand{\realset}{\mathbb{R}}

\newcommand{\abs}[1]{\left|#1\right|}

\newcommand{\I}[1]{\mathds{1} \! \left\{#1\right\}}

\newcommand{\set}[1]{\left\{#1\right\}}

\newcommand{\T}{^\top}
\newcommand{\eg}{\textit{e.g.}}
\newcommand{\spe}{\texttt{SafePE}}

\DeclareMathOperator*{\E}{\mathbb{E}}

\mathchardef\mhyphen="2D

\newcommand{\safeod}{\ensuremath{\tt SafeOD}\xspace}

\newcommand{\ie}{\textit{i.e.}}
\newcommand{\pie}{\pi_{e}}

\renewcommand{\st}{\textit{s.t.}}
\renewcommand{\Pr}{\text{Pr}}
\newcommand{\regret}{\texttt{Regret}}

\usepackage{natbib}
\bibpunct[, ]{(}{)}{,}{a}{}{,}%
\def\bibfont{\small}%
\usepackage{bm,comment}
\usepackage{appendix}

\ECRepeatTheorems

\EquationsNumberedThrough    

\allowdisplaybreaks
\begin{document}


\RUNAUTHOR{}

\RUNTITLE{Safe Exploration}

\TITLE{\Large Safe Data Collection for Offline and Online Policy Learning}

\ARTICLEAUTHORS{

\AUTHOR{Ruihao Zhu}
\AFF{Cornell University SC Johnson College of Business\\
	 \EMAIL{ruihao.zhu@cornell.edu}} 
 \AUTHOR{Branislav Kveton}
 \AFF{Amazon \\
 	\EMAIL{bkveton@amazon.com}} 
} 

\ABSTRACT{%
Motivated by practical needs of experimentation and policy learning in online platforms, we study the problem of \emph{safe data collection}. Specifically, our goal is to develop a \emph{logging policy} that efficiently explores different actions to elicit information while achieving competitive reward with a baseline \emph{production policy}. We first show that a common practice of mixing the production policy with randomized exploration, despite being safe, is sub-optimal in maximizing information gain. Then, we propose a safe optimal logging policy via a novel water-filling technique for the case when no side information about the actions' expected reward is available. We improve upon this design by considering side information and also extend our approaches to the linear contextual model to account for a large number of actions. 

Along the way, we analyze how our data logging policies impact errors in off(line)-policy learning and empirically validate the benefit of our design by conducting extensive numerical experiments with synthetic and MNIST datasets. To further demonstrate the generality of our approach, we also consider the safe online learning setting. By adaptively applying our techniques, we develop the Safe Phased-Elimination (\spe) algorithm that can achieve optimal regret bound with only logarithmic number of policy updates.
}%



\maketitle

%
\section{Introduction}
\label{sec:introduction}
Experimentation is used widely to identify new policies that efficiently allocate traffic to different actions. With ever-increasing demand for experiments, several companies have developed infrastructure to carry them out at scale (see, \eg, \cite{Optimizely,Google}). Among others, one of the most prevalent experimentation techniques is \emph{adaptive} online learning (\eg, multi-armed bandit \citep{ABF02,LS18}). In a nutshell, online learning dynamically adjusts the policy based on real-time feedback, and gradually allocates more and more traffic to better-performing actions. With its advantage in reducing experimentation cost, (near-)optimal online learning algorithms have been developed for many different applications, such as pricing \citep{KleinbergL03,KZ14}, hyperparameter tuning \citep{LiJDRT17,Quan21}, and recommendation \citep{LCLS10,FerreiraPS22}.

However, although a major progress has been made over the past decades, many challenges remain for an even broader adoption of online learning:
\begin{itemize}[leftmargin=*]
  \item \textbf{Challenge 1. Infrastructure:} To implement fully online learning algorithms in real world, it is necessary to collect responses and update traffic allocation in near real time, which poses significant challenges to the computational infrastructure \citep{GaoHRZ19,ChenCW20,LeviX21}.
  \item \textbf{Challenge 2. Logged-Data Estimation Error and Bias:} In many applications (\eg, ads design and webpage layout), it is important to understand the performance of every action, even those that under-perform \citep{Danilchik20}. However, due to its cost-minimizing nature, online learning algorithms adaptively allocate less traffic to actions with poor historical performance. Therefore, it is common to encounter a major estimation error when estimating their expected reward from the logged data. Even worse, existing works (see, \eg, \cite{NieTTZ17,ShinRR19}) showed that a direct application of maximum likelihood estimation to adaptively collected data can result in a significant bias. The de-biasing is challenging because the data-logging policy is adapted over time to the collected data.
  \item \textbf{Challenge 3. Safety:} Non-Bayesian online algorithms tend to explore extensively in the initial rounds. This can have a major impact on user experience and lead to early termination of the experiment \citep{WuSLS16,BastaniHPS21}.
\end{itemize}

To alleviate the workload of the infrastructure (Challenge 1), practitioners have proposed a refined experimentation scheme, which first collects data by deploying a \emph{static} logging policy, and then utilizes the logged data to learn new policies offline and/or to make un-biased inference if needed (Challenge 2). Finally, the policy is updated based on the learned knowledge (see \cref{fig:policy} for a illustration). In this scheme, the process of learning the policies offline is known as \emph{off-policy learning} \citep{DudikELL14,SwaminathanJ15}, where the performance of a policy is estimated (via the logged data) without deploying it. Depending on the application, this scheme (and its variants) may also be applied iteratively in an online fashion to further reduce experimentation cost with the benefit of low adaptivity (see \eg, \cite{LeviX21,GaoHRZ19}). 
\begin{figure}[!ht]
	\includegraphics[width=9.5cm,height=2.1cm]{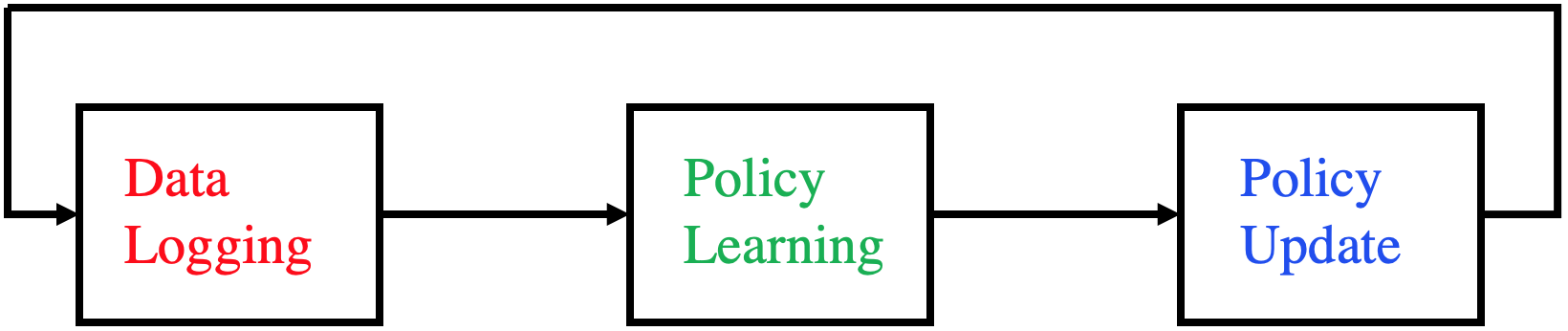}
	\centering
	\caption{Overview of the scheme of developing new policies.}
	\label{fig:policy}
\end{figure}

Off-policy learning crucially relies on sufficiently explored logged data to drawn conclusions about candidate policies. When the data are collected, it is typically necessary to satisfy \emph{safety constraints}, which prohibit excess exploration and too costly experimentation (Challenge 3). To strike the balance, a common practice in the industry is to mix a baseline \emph{production policy} with randomized actions. This results in a logging policy that explores, as it allocates traffic to all actions, but is also safe because the production policy is followed frequently. As an example, if the logging policy has to perform as well as 95\% of the production policy, then 95\% of the traffic is allocated to the production policy, while the rest is randomly allocated to all actions. Whether the logging policy is statistically efficient and suitable for the goal is rarely questioned. Ironically, the performance of this scheme critically depends on the quality of the logged data. This raises an important question of \emph{how to design a logging policy that is both safe and collects high-quality data}.

\subsection{Main Contributions}
 In this work, we make the following contributions:
 
\begin{itemize}
	\item \textbf{Modeling:} To quantify the quality of logged dataset, we study this question through the lens of \emph{G-optimal design} (\ie, globally-optimal design \citep{KieferW60}). In the G-optimal design, the goal is to design a data logging policy that minimizes (a proxy of) the maximal variance in estimating each and every action's expected reward. Motivated by practical safety considerations in experimentation \citep{WuSLS16}, we instantiate the safety constraint as follows: the expected reward of the logging policy is at least an $\alpha$ fraction of that of the production policy. 
	\item \textbf{Optimal Solution:} We first show, perhaps surprisingly, that the common practice of mixing production policy with uniform exploration is sub-optimal. Then, we propose a water-filling algorithm that solves our problem optimally when no side information about the actions' expected reward is available. We improve upon this design by considering side information, and also extend both approaches to a large number of actions with a linear reward model. 
	\item \textbf{Applications in Offline and Online Policy Learning:} In \cref{sec:ope}, we first apply our approach to off-policy evaluation and optimization, and show that our logging policy can provide performance guarantees for the asymptotically optimal inverse propensity score (IPS) estimator \citep{DudikELL14}. We conduct extensive numerical experiments with both synthetic and MNIST \citep{lecun10mnist} datasets to demonstrate the performance of our approaches. In \cref{sec:explore}, to further demonstrate the generality of our approach, we also apply it adaptively to the setting of safe online learning \citep{WuSLS16}, where our goal is to maximize the expected cumulative reward while respecting the safety constraint. Specifically, we propose a novel Safe Phased-Elimination (\spe) algorithm that achieves optimal (up to logarithmic factors) regret, but only with a logarithmic number of policy updates. This helps to significantly reduce the frequency of policy updates compared to prior works, and thus the workload of the infrastructure.
	\item\textbf{Preliminary Version:} A preliminary version \citep{ZhuK22} of this work was published at the 25th International Conference on Artificial Intelligence and Statistics. Compared to \cite{ZhuK22}, in this work
	\begin{enumerate}
	    \item We improve the off-policy learning guarantee in \cref{lemma:ips_error_bound} by further removing a factor of $\sqrt{K},$ where $K$ is the number of actions. This bound appears to be novel and attains state-of-the-art convergence rate;
	    \item We also conduct off-policy evaluation experiments with the MNIST dataset in \cref{sec:mnist dataset};
	    \item More importantly, in \cref{sec:explore}, we additionally apply our approach to safe online learning where more efficient data collection and policy learning is made possible through adaptivity. For this setting, we recover the optimal regret bound, but only using a low number of policy updates.
	\end{enumerate} 
\end{itemize} 
\subsection{Additional Related Works}
\label{sec:related work}
In this section, we review the connections between exploration, learning, and operations as well as prior works in safe exploration. 

\noindent\textbf{Exploration and Learning in Operations:} Online learning algorithms have been widely used in pricing and revenue management. For instance, \cite{KZ14,BK18,ZhuZ20} developed phased exploration type algorithms when the demand model is parametric. In \cite{WangCSL21,ChenG21}, the authors considered pricing problems with non-parametric demand models. More recently, a couple of works have also tried to devise optimal learning algorithms to solve the problem of pricing and inventory control jointly \citep{ChenCA19,LiZ20,ChenWZ20,KeskinLS21}. We remark that the above mentioned works mainly focus on minimizing regret, and safety was not considered (Challenge 3). More importantly, as discussed above (Challenge 2), estimation based on the dataset collected by online learning methods is prone to large estimation error and bias. In contrast, our (forthcoming) formulation and solution address these two concerns.

Another line of works focuses on pure exploration, where the goal is to explore different actions as much as possible in order to infer the reward of them. Among others, \cite{BubeckMS10,JamiesonN14,XiongABI19} considered exploration in stationary environments. More recently, \cite{WuZZZW22} studied how to explore in a constantly changing and thus non-stationary environment. Compared to these works, our safe data collection setting and the corresponding solution similarly benefit the downstream inference tasks by mitigating estimation bias with non-adaptively collected data. But critically, we further introduce the safety consideration, which would lead to drastically different solution.

\noindent\textbf{Safety in Learning:} \citet{WuSLS16} proposed a bandit algorithm that conservatively improves upon a default action. The key idea is to take the default action $\alpha$ fraction of time and improve it over time, with provably better actions with a high probability. This work was generalized to linear bandits by \citet{NIPS2017_bdc4626a} and to combinatorial action spaces, such as in online learning to rank, by \citet{pmlr-v115-li20b}. In a recent work, \cite{XUXBB21} extended the framework to episodic reinforcement learning. Our work is similar to these works by considering a similar safety constraint. However, there are two critical differences: 1) In developing the safe optimal design (in the forthcoming \cref{sec:simplex} and \cref{sec:general}), we learn the most exploratory policy under a safety constraint that collects useful data for future off-policy estimation and optimization; 2) When applied to safe online learning, our \spe~algorithm only requires a logarithmic number of policy changes, which is exponentially fewer than in the existing works. This makes it more suitable for practical use as it does not impose demanding requirements on the infrastructure.

Another popular problem is off-policy optimization with a safety constraint, where the learned policy improves over a logging policy with a high probability \citep{pmlr-v37-thomas15,pmlr-v97-laroche19a}. These works solve an orthogonal problem to ours. They learn policies with enough support to improve over the logging policy, while we explore to collect better data for future off-policy estimation and optimization.

\section{Problem Formulation}
\label{sec:optimal_designs}
\noindent\textbf{Notations:} Let $\cA = [K]: = \set{1, \dots, K}$ be a tabular \emph{action set}. When action $a \in \cA$ is taken, we observe its stochastic reward with (initially) unknown mean $\bar{r}(a)\in[0,1].$ A \emph{policy} $\pi:\cA\to[0,1]$ is a probability distribution on $\cA$ and we denote by $\Pi$ the set of all possible policies. Following Section 21.1 of \cite{LS18}, we use the terms ``policy" and ``design" interchangeably.  To simplify notation, we use $\bar{r}$ and $\pi$ to denote the vectorized expected reward and the policy, \ie, $\bar{r}=(\bar{r}(1), \dots, \bar{r}(K))\T$ and $\pi= (\pi(1), \dots, \pi(K))\T.$ The expected reward of policy $\pi$ is thus $V(\pi;\bar{r})=\pi^{\top}\bar{r}$ (when the context is clear enough, we may also suppress the explicit dependence on $\bar{r}$). For any $p\geq 0,$ we define $\|\cdot\|_p$ as the $p$-norm and $\Delta_{k-1}=\{x\in\realset^{k}:x\geq 0,\|x\|_1= 1\}$ as the $k$-dimensional simplex. For any $c\in\realset,$ we use  $c\bm{1}_k$ to denote the $k$-dimensional vector with all entries equal to $c$ (when $c=0,$ we write this as $\bm{0}_k$). We use $\I{\cdot}$ to denote the indicator function. For two vectors $a,b$ with same dimension, we use $a\circ b$ to denote their Hadamard (coordinate-wise) product. A random variable $X$ is $\sigma^2$-sub-Gaussian if $\E[X] = 0$ and its moment generating function satisfies $\E[\exp(sX)]\leq\exp(s^2\sigma^2/2)$ for all $s\in\realset.$

\noindent\textbf{Tabular Safe Data Collection Setup:} To overcome the challenges posed by online learning-based experimentation (see \cref{sec:introduction}), we deploy a static exploratory \emph{data logging policy} $\pie$ for a certain time interval. For each time step of this interval, we randomly select an action according to $\pie$ and observe the corresponding realized random reward. We seek to leverage the collected data to estimate each action's expected reward $\bar{r}(\cdot)$ and to further identify the optimal policy $\pi_*=\argmax_{\pi\in\Pi}V(\pi)$ offline (see \cref{sec:ope}) or to maximize our cumulative reward (see \cref{sec:explore}). Before formally introducing our objective, we first describe our criteria in developing $\pie:$
\begin{itemize}
\item\textbf{Information Gain:} The quality of our logging policy $\pie$ is measured by $g(\pi) = \max_{a\in\cA}1/\pi(a).$ Intuitively, $g(\pi)$ is a variance proxy of the data collected by $\pi$ and it is proportional to the maximum width of a (high-probability) confidence interval over $a \in \cA$ (see \eg, Section 21.1 of \cite{LS18}). Thus it measures how well we can estimate the unknown expected reward and compute a near-optimal policy. A sensible objective is to find $\pie$ that minimizes $g(\pie)$. Note that $g(\pi)$ is a special case of the G-optimal design objective \citep{KieferW60} and without any constraint, we can set $\pie(a) = 1 / K$ for all $a\in\cA$ to maximize information gain.
\item\textbf{Safety:} To avoid a potentially high cost in deploying $\pie$, we demand that $\pie$'s expected reward is at least $\alpha \in[0,1]$ of that of a \emph{baseline production policy} $\pi_0$ for any instance of expected reward $\bar{r}.$ Specifically, $V(\pie;\bar{r})\geq\alpha V(\pi_0;\bar{r})$ for all $\bar{r},$ where $\alpha\in[0,1]$ is a \emph{safety parameter}. We remark that 
\begin{enumerate}
    \item The safety constraint could be defined alternatively as that $\pie$'s expected reward is at most $\alpha$ less than that of $\pi_0$'s, \ie, $V(\pie;\bar{r}) \geq V(\pi_0;\bar{r}) - \alpha$;
    \item We do not impose that $\pi_0$ is deterministic as it might have to satisfy additional constraints (\eg, resource capacity, fairness, etc.) and cannot keep selecting a single action.
\end{enumerate} 
Nevertheless, our forthcoming results could easily incorporate the above two points as well.
\end{itemize}

\noindent\textbf{Objective:} Formally, we want to design $\pie$ that simultaneously collects high-quality data to maximize information gain and ensures safety. Therefore, our problem is
\begin{align}
  \min \ &g(\pie) \nonumber \\
  \mathrm{s.t.} \
  & \pie\in\Delta_{K-1} \nonumber \\ 
  &\min_{\bar{r}}V(\pie;\bar{r})-\alpha V(\pi_0;\bar{r})
  \geq0 \,.
  \label{eq:tabular_no_side}
\end{align}
To instantiate the second constraint of \eqref{eq:tabular_no_side}, we distinguish two cases based on prior information about $\bar{r}$:
\begin{itemize}
    \item \textbf{No Side Information:} When a brand new experiment is carried out, we have no information about $\bar{r}.$ In this case, we assume no extra information about $\bar{r}$ except for being bounded, \ie, $\bar{r}\in[0,1]^K.$
    \item \textbf{Side Information:} Thanks to historical data from past experiments, prior information about $\bar{r}$ is often available in the form of probabilistic prior \citep{BastaniSLZ21,KvetonKZH21,SimchowitzTK21} or confidence intervals \citep{ZhangJM20}. In this case, we assume that side information about $\bar{r}$ is given as confidence intervals (as this can also be constructed with a given prior), \ie, $\forall a\in\cA,$ $\bar{r}(a)\in[L(a),U(a)]~(\subset[0,1])$. We remark that, for now, we treat this constraint as deterministic except. A high-probability treatment would be analogous except that the claims would hold in with high probability. We demonstrate this in \cref{sec:explore}, where we show that by properly incorporating side information, one can achieve optimal regret in safe online learning.
\end{itemize}

\subsection{Mixing with Uniform Exploration is Sub-Optimal in General}
\label{sec:mixing}
We first show that even a simple variant of our problem has an interesting structure. Specifically, we take the \emph{no side information} case as an example, and show that mixing of the uniform exploration distribution with the production policy is generally sub-optimal.

\noindent\textbf{Mixing with Uniform Exploration:} As indicated by its name, this heuristic would follow $\pi_0$ for $\beta$ fraction of the time while uniformly sample all the actions otherwise. Formally, the policy is defined as 
\begin{align}\label{eq:mixture}
{\pi}_{\beta}:=\beta\pi_0+\frac{(1-\beta)\mathbf{1}_K}{K}.
\end{align} This is a commonly used strategy for multi-armed bandit (see \eg, Section 1.2.1 of \cite{Slivkins19}), reinforcement learning (see \eg, Section 2.2 of \cite{SuttonB18}), and conservative online exploration \citep{WuSLS16,YangWZG21}. 

\noindent\textbf{Balance the Amount of Exploration:} Suppose w.l.o.g. that $\pi_0(1)\leq\ldots\leq\pi_0(K).$ For any $\beta\in[0,1],$ it is easy to verify that $\pi_{\beta}(1)\leq\ldots\leq\pi_{\beta}(K)$ and $g(\pi_{\beta})=\pi_{\beta}(1)^{-1}.$ Since $\pi_0(1)\leq 1/K,$ it is evident that a smaller $\beta$ would lead to a smaller $g(\pi_{\beta}).$ However, we may not be able to set $\beta=0$ due to the safety constraint. To satisfy the safety constraint, we need to enforce that
\begin{align}\label{eq:safety_equivalent}
\pi_{\beta}(a)\geq \alpha \pi_0(a)\qquad \forall a\in\cA.
\end{align} 
This is because if there exists an action $a\in\cA$ such that $\pi_{\beta}(a)<\alpha \pi_0(a)$, then the safety constraint can be easily violated by setting $\bar{r}(a')=0$ for all $a'\in\cA\setminus\{a\}$. To this end, by solving the inequalities $(1-\beta)\pi_0(a)+\beta/K\geq\alpha\pi_0(a)$ for all $a,$ we get that $$\beta\geq\beta_*:=\max\left\{\frac{\alpha-(K\pi_0(K))^{-1}}{1-(K\pi_0(K))^{-1}},0\right\}.$$
It is evident that $\beta_*$ depends on $\pi_0$ through $\pi_0(K).$ Intuitively, this is because as we decrease $\beta$, the safety constraint is violated first for the most frequently taken action. At that point, we know that $\alpha \pi_0(K) = \beta_* \pi_0(K) +(1- \beta_*)/K$. Based on this construction, the following example shows that $\pi_{\beta_*}$ is not always optimal.

\begin{example}
Let $K=3$, the production policy be $\pi_0=(0.1,0.3,0.6)^{\top},$ and the safety parameter be $\alpha=0.8$. Then $\beta_*=0.55$ and $\pi_{\beta_*}=(0.205,0.315,0.48)^{\top}.$ Now consider the policy $\pi=(0.26,0.26,0.48)^{\top}.$ We can verify that the safety constraint is satisfied as $\pi(a)\geq\alpha \pi_0(a)$ for all $a\in\cA.$ But we have $$g(\pi_{\beta_*})={0.205}^{-1}>{0.26}^{-1}=g(\pi),$$ and thus $\pi_{\beta_*}$ is sub-optimal.
\end{example}

The above example shows that mixing of the production policy with a uniform distribution yields a sub-optimal logging policy. In \cref{sec:supplement}, we show that $\pi_{\beta}$ would be sub-optimal if $\pi_0$ takes more than two values and $\alpha$ is above a certain threshold (\ie, when the safety constraint is not too loose) while it would be optimal otherwise.

\section{Tabular Safe Optimal Design}
\label{sec:simplex}
Motivated by our discussions in \cref{sec:mixing}, we introduce our solutions based on safe optimal designs with and without side information. We start with the so-called tabular case.

\subsection{Safe Optimal Design Without Side Information}
\label{sec:simplex_no_side}
We note that $\pi_{\beta}$ in \cref{eq:mixture} of \cref{sec:mixing} is sub-optimal because the peeled-off probability mass from $\pi_0$ is added uniformly to all actions instead of those with the lowest probabilities, so as to reduce $g$ maximally. Thus we consider a more direct \emph{water-filling} method that first takes $(1-\alpha)$ portion mass off from each $\pi_0(a)$ to form $\pi'(a)$ without violating the safety constraint, \ie, $\pi'=\alpha\pi_0,$ and then re-allocate the peeled-off mass to $\pi'$ in a greedy manner. That is, as shown in \cref{fig:wf}, it successively increases the probability mass of the actions with the lowest probabilities in $\pi_0$ until all the $(1-\alpha)$ probability mass is exhausted. 

\begin{figure}[!ht]
\includegraphics[width=9cm,height=3.5cm]{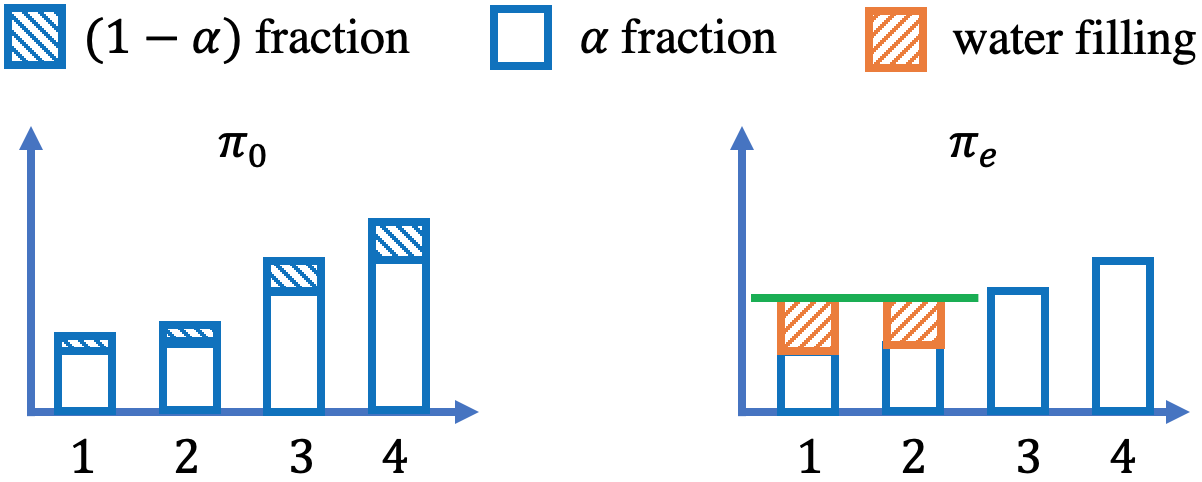}
\centering
\caption{Water-filling method}
\label{fig:wf}
\end{figure}

\noindent\textbf{Water-Filling Method:} Assuming w.l.o.g. that $\pi'(1)\leq\pi'(2)\leq\ldots\leq\pi'(K),$ 
the algorithm searches for the largest $k\in[K]$ such that 
$k\cdot\pi'(k)+\sum_{i=k+1}^K\pi'(i)\leq 1,$ and then sets
$\pie(i)=(1-\sum_{i=k+1}^K\pi'(i))/k$ for all $i\leq k$ and $\pie(i)=\pi'(i)$ for all $i > k.$ Now we establish that the water-filling method is optimal.

\begin{theorem}\label{thm:greedy}
For any policy $\pi$ that satisfies the safety constraint, we have $\min_{a}\pi_e(a)\geq\min_{a}{\pi}(a)$ in the no side information case, \ie, $\bar{r} \in [0, 1]^K$.\end{theorem}
\begin{proof}[Proof Sketch] We prove this claim by contradiction. By virtue of water filling, the actions in $\cA$ could be divided into two groups: those that satisfy $\pie(a)=\alpha\pi_0(a),$ \ie, they do not receive any additional mass during the water-filling step; and those that receive additional mass during the water-filling step and their probabilities become the same. Critically, probabilities of all actions in the second group are equal to $\min_{a\in\cA}\pie(a)$. Now, for a policy $\pi$ to have $\min_{a\in\cA}\pi(a)>\min_{a\in\cA}\pie(a),$ it must be $\pi(a)<\alpha\pi_0(a)$ for some $a$ in the first group (because if $\pi(a)<\pie(a)$ for some $a$ in the second group, $\min_{a\in\cA}\pi(a)<\min_{a\in\cA}\pie(a)$), which would violate the safety constraint (when $\bar{r}(i)=\I{i=a}$). The complete proof is provided in \cref{sec:thm:greedy}.
\end{proof}

\subsection{Safe Optimal Design With Side Information}
\label{sec:simplex_side}
Now we turn to the case with side information. The side information gives us more flexibility in satisfying the safety constraint. Notably, now $\pie(a)<\alpha\pi_0(a)$ can happen for some actions $a$ as long as $\pie$ allocates enough probability to actions with high expected reward to compensate for this deficit. The water-filling method in \cref{sec:simplex_no_side} does not solve this problem optimally anymore. Instead, we formulate the problem of finding the optimal policy $\pie$ as 
\begin{align*}
  \text{P}_1(L,U,\pi_0):\quad\max\ &\gamma\quad\\
  \mathrm{s.t.} \
  &\pie \geq \gamma\mathbf{1}_K\,,\ \pie \in \Delta_{K - 1} \\
  &\min_{\bar{r}\in[L,U]}(\pie - \alpha\pi_0)\T \bar{r} \geq 0\,.
\end{align*}
Here $\gamma$ is a tight lower bound for $\min_a \pie(a)$ and by maximizing $\gamma,$ we equivalently minimize $g(\pie).$ The last constraint enforces that $V(\pie)\geq\alpha V(\pi_0)$ holds for all possible $\bar{r}\in[L,U] \subseteq [0, 1]^K.$ Note that when $[L,U]=[0,1]^K,$ we can recover the solution of the water-filling method for the no side information case (\cref{sec:simplex_no_side}).

One challenge posed by $\text{P}_1(L,U,\pi_0)$ is that its last constraint implicitly contains infinitely many constraints. These constraints can be satisfied incrementally using the cutting-plane method (see, \eg, Chapter 6.3 of \cite{BertsimasT97}). More elegantly though, motivated by robust optimization \citep{TalGN09}, we consider the following sub-optimization problem  based on the last constraint
\begin{align*}
    \text{P}_2(L,U,\pi_0,\pie):\quad\min\ &(\pie - \alpha \pi_0)\T \bar{r} \\
    \mathrm{s.t.} \ &L\leq\bar{r}\leq U
\end{align*}
and its dual 
\begin{align*}
   \text{D}_2(L,U,\pi_0,\pie):\quad\max\ &L^{\top}z_1-U^{\top}z_2\quad \\
   \mathrm{s.t.} \ &z_1-z_2=\pie-\alpha\pi_0\,,\ z_1,z_2\geq \bm{0}\,,
\end{align*} 
where $z_1$ and $z_2$ are $K$-dimensional vectors serving as dual variables. Since $\text{P}_2(L,U,\pi_0,\pie)$ has a finite optimal value, by strong duality (see \eg, chapter 4 of \cite{BertsimasT97}), we have that the optimal objective values of $\text{P}_2(L,U,\pi_0,\pie)$ and $\text{D}_2(L,U,\pi_0,\pie)$ are the same. Thus $\text{P}_1(L,U,\pi_0)$ can be equivalently written as 
\begin{align*}
  \text{P}_3(L,U,\pi_0):\quad\max\ &\gamma\quad\\
  \mathrm{s.t.} \
  &\pie \geq \gamma\mathbf{1}_K\,,\ \pie \in \Delta_{K - 1}\\ &L^{\top}z_1-U^{\top}z_2 \geq 0\,,\ z_1-z_2=\pie-\alpha\pi_0\,,\
  z_1,z_2\geq \bm{0}\,.
\end{align*}
Intuitively, using the duality between $\text{P}_2(L,U,\pi_0,\pie)$ and $\text{D}_2(L,U,\pi_0,\pie),$ we translate the minimization problem in the last constraint of $\text{P}_1(L,U,\pi_0)$ to a maximization problem. As a consequence, instead of checking whether $\pie$ satisfies $(\pie-\alpha\pi_0)^{\top}\bar{r}\geq0$ for all possible $\bar{r}\in[L,U],$ one only needs to find a single pair $z_1,z_2$ that satisfies the last three constraints in $\text{P}_3(L,U,\pi_0)$. Therefore, $\text{P}_3(L,U,\pi_0)$ is a linear program that can be solved directly. 

Following the duality argument above, the equivalence of $\text{P}_1(L,U,\pi_0)$ and  $\text{P}_3(L,U,\pi_0)$ can be established. For completeness, we include the proof of the following theorem in \cref{sec:thm:duality}.
\begin{theorem}\label{thm:duality}
The optimal value of $\text{P}_1(L,U,\pi_0)$ is equal to the optimal value of $\text{P}_3(L,U,\pi_0).$
\end{theorem}

\begin{remark}\label{remark:tabular_side}
An alternative way of solving $\text{P}_1(L,U,\pi_0)$ follows from the observation that, in the last constraint of $\text{P}_1(L,U,\pi_0),$ the minimum is attained at either $\bar{r}(a)=L(a)$ (if $\pie(a)-\alpha\pi_0(a)\geq0$) or $U(a)$ (if $\pie(a)-\alpha\pi_0(a)<0$). We can thus introduce a variable $z\in\realset^K$ to serve as a coordinate-wise lower bound for $(\pie-\alpha\pi_0)\circ L$ and $(\pie-\alpha\pi_0)\circ U,$ and mandate that $z^{\top}\mathbf{1}_K\geq0$ to ensure $\min_{\bar{r}\in[L,U]}(\pie-\alpha\pi_0)^{\top}\bar{r}\geq0.$ Consequently, we can rewrite $\text{P}_1(L,U,\pi_0)$ as $\text{P}'_1(L,U,\pi_0):$
\begin{align*}\max\ &\gamma \quad\\
  \mathrm{s.t.} \
  &\pie \geq \gamma\mathbf{1}_K\,, \pie \in \Delta_{K - 1} \\ 
  &z^{\top}\mathbf{1}_K\geq0\,,\ (\pie-\alpha\pi_0)\circ L\geq z\,, \
  (\pie-\alpha\pi_0)\circ U\geq z\,.
\end{align*} 
The equivalence of $\text{P}_1(L,U,\pi_0)$ and $\text{P}'_1(L,U,\pi_0)$ is formally established in \cref{sec:remark:tabular_side}. 
\end{remark}

\section{Linear Safe Optimal Design}\label{sec:general}
So far we assumed a tabular action set $\cA,$ where the expected reward of actions are unrelated. While this setting is suitable for a small number of actions, the performance (\ie, the objective function $g$) would quickly deteriorate if $|\cA|$ was large. The reason is that if no correlations exist among the expected reward, $\min_{a\in\cA}\pi(a)\leq1/|\cA|$ as $\sum_{a\in\cA}\pi(a)=1$, and hence $g(\pi)\geq |\cA|$ even without any safety constraints. This essentially implies that if we apply our tabular methods to a large action set, the collected dataset would lead to major estimation error for policy learning (see the forthcoming \cref{sec:ope} and \cref{lemma:ips_error_bound} therein). Even worse, in practice, $|\cA|$ is expected to be large in many popular applications, such as the large pool of ads in online advertising \citep{LCLS10,CLRS11} or the combinatorial action space in online recommendations \citep{SwaminathanKAD17,VlassisCGK21}. 

To address the challenge of large action sets, prior works relied on features \citep{AYPS11,SwaminathanKAD17} or exploits the structures of the action sets \citep{RadlinskiKJ08,KvetonSWA15}. Here, we take the former approach and follow the linear function approximation scheme, where the expected reward of each action is linear in the action's features and an underlying shared reward parameter \citep{AYPS11,SwaminathanKAD17}. We adopt this approach and generalize our results to the linear function approximation. Let $\cA \subset \realset^d$ be the action set that contains a collection of $d$-dimensional feature vectors with $\|a\|_2\leq 1~\forall a\in\cA.$ For any logging policy $\pi:\cA\to[0,1],$ we generalize $g(\pi)$ in \cref{sec:optimal_designs} to
\begin{align}
  g(\pi)
  = \max_{a\in\cA}
  a\T G(\pi)^{-1} a\,,
  \label{eq:g-optimal design}
\end{align}
where $G(\pi)=\sum_{a\in\cA} \pi(a) a a\T.$ We remark that in the tabular case, $\cA$ would be the standard Euclidean basis.

Similarly to the tabular case, $g$ is proportional to the maximum width of a high-probability confidence interval over $a \in \cA$ (see \eg, Section 21.1 of \cite{LS18}). Our goal is to design a logging policy $\pie$ that minimizes $g(\pie)$, so as to minimize our estimation error. In \emph{absence} of the safety constraint, this is the general form of the G-optimal design \citep{KieferW60}, which is a convex optimization problem that can be solved efficiently by the Frank-Wolfe algorithm (see \eg, \cite{Fedorov72}). To describe the safety constraint, we let $\theta_*\in\cS_{d-1} = \{\theta\in\realset^d:\|\theta\|_2\leq 1\}$ be an unknown parameter vector and  $\bar{r}(a)=a\T\theta_*~(\in[0,1])$ be the expected reward of action $a.$ Then the safety constraint would require that $ V(\pie)\geq\alpha V(\pi_0)$ for all $\theta_*\in\cS_{d-1}.$  

\noindent\textbf{Side Information:} In linear models, confidence intervals on $\theta_*$ are often given in the form of ellipsoids (see, \eg, \cite{AYPS11,BK18} or Chapter 20 of \cite{LS18}). We consider this generalization here, by assuming that the unknown parameter $\theta_*$ falls in a confidence ellipsoid (possibly with a high probability)
$$
  \Theta:
  = \set{\theta \in \realset^d: (\theta - \bar{\theta})\T
  \bar{\Sigma}^{-1} (\theta - \bar{\theta}) \leq1}.
$$
Here $\bar{\theta}\in\realset^d$ is the center of the ellipsoid and $\bar{\Sigma}^{-1}\in\realset^{d\times d}$ is a positive definite matrix whose eigenvectors are the directions of the principal semi-axes of the ellipsoid and whose eigenvalues are the reciprocals of the squares of the lengths of the semi-axes.
\begin{remark}
We point out that if we consider no side information (\cref{sec:simplex_no_side}) or the coordinate-wise side information (\cref{sec:simplex_side}), we can apply the results from \cref{sec:simplex} to compute the optimal designs.
\end{remark}

To ease exposition, we assume that $\Theta\subset\cS_{d-1}$ and define a $d \times K$ matrix $A=(a^{(1)},\ldots,a^{(K)}),$ where $a^{(i)}$ is the $i$-th action in $\cA$. Then the safety constraint can be written as
$
  \pie^{\top} A^{\top} \theta_*
  \geq \alpha\pi_{0}^{\top} A^{\top} \theta_*~~\forall \theta_* \in \Theta,$
and the problem of finding the optimal logging policy that satisfies the safety constraint is 
\begin{align}\label{eq:linear_side}
  \nonumber\text{P}_4(\Theta,\pi_0):\quad\min\ 
 &g(\pie)\quad\\\st \ \nonumber&\pie\in\Delta_{K-1}\\
 &\min_{\theta_*\in\Theta}(\pie-\alpha\pi_0)^{\top}A^{\top}\theta_*\geq0\,.
\end{align} 

\subsection{Exact Solution and Computational Bottleneck}\label{sec:duality_linear}
As in the tabular case, the last constraint of $\text{P}_4(\Theta,\pi_0)$ also requires the inequality to hold for a continuum of $\theta_*$, and hence implicitly consists of infinitely many constraints. We could follow the duality approach in \cref{sec:simplex_side}, to convert $\text{P}_4(\Theta,\pi_0)$ to a convex optimization problem with a quadratic constraint. Specifically, since $\Theta$ is an ellipsoid, we derive a closed-form expression for this constraint via the Lagrangian multiplier method (included in \cref{lemma:lagrangian} of \cref{sec:auxiliary} for completeness) as
\begin{align}
	\nonumber\min_{\theta_*\in\Theta}(\pie-\alpha\pi_0)^{\top}A^{\top}\theta_*=-\max_{\theta_*\in\Theta}(\alpha\pi_0-\pie)^{\top}A^{\top}\theta_*=(\pie-\alpha\pi_0)^{\top}A^{\top}\bar{\theta}-\sqrt{(\pie-\alpha\pi_0)\T A\T\bar{\Sigma}A(\pie-\alpha\pi_0)}\,.
\end{align}
Then the optimization problem $\text{P}_4(\Theta,\pi_0)$ can be simplified as
\begin{align}
	\nonumber\text{P}_5(\Theta,\pi_0):\quad&\min\ g(\pie)\\
	\nonumber\quad&\st \ \pie\in\Delta_{K-1}\\
	\nonumber&\sqrt{(\pie-\alpha\pi_0)\T A\T\bar{\Sigma}A(\pie-\alpha\pi_0)}\leq(\pie-\alpha\pi_0)\T A\T\bar{\theta}\,.
\end{align}
It is now evident that both the objective and constraints are convex in $\pie$. Therefore, one way to solving this problem would be to apply the idea of online convex optimization followed by the trick of online-to-batch conversion (see, \eg, chapter 3 of \cite{Orabona19}). Note that $\text{P}_5(\Theta,\pi_0)$ has a quadratic constraint, which implies solving it directly via conventional iterative convex optimization algorithm (\eg, gradient descent) would still be computationally challenging. This is because we would need a computationally expensive projection step, which casts the updated intermediate solution back to the feasible region, upon each update.

\subsection{Frank-Wolfe with a Cutting Plane Method}
\label{sec:frank-wolfe}
We solve problem $\text{P}_4(\Theta,\pi_0)$ without projections by using a Frank-Wolfe algorithm \citep{FrankW56} with the cutting-plane method (see, \eg, chapter 6.3 of \cite{BertsimasT97}).

The algorithm is iterative and we denote its output after iteration $i$ by ${\pi}^{(i)}$. The initial solution ${\pi}^{(0)}$ is set to $\pi_0$. In each iteration, the Frank-Wolfe algorithm proceeds by minimizing a linear approximation of the objective function and sets $\pi^{(i+1)}$ to its minimizer. 

\noindent\textbf{Linearization:} More formally, we use $H(\pi)$ to denote the gradient of $\max_{a\in\cA}a^{\top}G(\pi)^{-1}a$ at $\pi.$ At the beginning of each iteration $i,$ the Frank-Wolfe algorithm considers the following linear program 
\begin{align*}
 \text{P}_6(\Theta,\pi_0):\quad \min \quad &\mathring{\pi}^{\top}H\left(\pi^{(i-1)}\right)\quad\\
   \mathrm{s.t.} \  &\mathring{\pi}\in\Delta_{K-1}\\ &\min_{\theta_*\in\Theta}(\mathring{\pi}-\alpha\pi_0)^{\top}A^{\top}\theta_*\geq0\,.
\end{align*}
Let $\mathring{\pi}^{(i)}$ be the optimal solution to the above linear program. Then we set $${\pi}^{(i)}=\pi^{(i-1)}+\eta(\mathring{\pi}^{(i)}-\pi^{(i-1)}),$$ where $\eta\in[0,1]$ is chosen (\eg, via line search) such that $\max_{a\in\cA}a^{\top}G({\pi}^{(i)})^{-1}a$ is minimized. The final output of this algorithm is $\pie.$
\begin{remark}
	We note that a different version of the Frank-Wolfe algorithm is implemented in chapter 21.2 of \cite{LS18} (see note 3) and its $\eta$ is computed in a closed form. This is because in the unconstrained setting, \ie, without safety constraint, one can choose to increase $\pie(a)$ for an arbitrary $a\in\cA$ without violating the safety constraint. However, in our case, changing $\pie(a)$ in this way could violate the safety constraint.
\end{remark}

\noindent\textbf{Computing the Gradient:} To work out $H(\pi),$ we compute the partial derivative of the objective function w.r.t. $\pi(a)$ as
\begin{align*}
	& \frac{\partial \max_{a \in \cA} a\T G\left({\pi}\right)^{-1} a}{\partial \pi(a)} 
	  = a_{\max}^{\top}\frac{\partial G(\pi)^{-1}}{\partial \pi(a)} a_{\max}
	  = - a_{\max}^{\top} G(\pi)^{-1} a a\T G(\pi)^{-1} a_{\max} = - \left(a_{\max}^{\top} G(\pi)^{-1} a\right)^2\,,
\end{align*}  
where $a_{\max} = \argmax_{a \in \cA} a\T G(\pi)^{-1} a$ is the action that achieves the maximum for a given policy $\pi$. The first equality follows from the fact that $a_{\max}$ is the maximizer under $\pi$. The second equality combines the derivative of matrix inverse with the fact that $G(\pi)$ is linear in $\pi$. Consequently,
$$H(\pi)= - \big(\left(a_{\max}^{\top} G(\pi)^{-1} a^{(1)}\right)^2, \ldots, \left(a_{\max}^{\top} G(\pi)^{-1} a^{(K)}\right)^2\big)^{\top}.$$ 
Here, we recall that $a^{(i)}$ is the $i$-th action in $\cA$.

\noindent\textbf{Dealing with Infinitely Many Constraints:} As before, the last constraint of $\text{P}_6(\Theta,\pi_0)$ implicitly includes infinitely many constraints. To address this, we generate the constraints incrementally using the cutting-plane method in each iteration $i$. Specifically, we start with $S$ as the empty set and denote by $\mathring{\pi}^{(i)}_S$ the corresponding optimal solution to $\text{P}_6(S,\pi_0)$. For a given $S$, we find the most violated constraint in $\Theta$ using
\begin{align*}
  \theta_S
  & = \argmin_{\theta_* \in \Theta} (\mathring{\pi}^{(i)}_S-\alpha\pi_0)^{\top}A^{\top}\theta_* = \bar{\theta}+\frac{\bar{\Sigma}A(\alpha\pi_0-\mathring{\pi}^{(i)}_S)}{\sqrt{(\alpha\pi_0-\mathring{\pi}^{(i)}_S)^{\top}A^{\top}\bar{\Sigma}A(\alpha\pi_0-\mathring{\pi}^{(i)}_S)}}\,.
\end{align*} 
The above closed-form solution follows from the fact that this problem is equivalent to maximizing a linear function on an ellipsoid; and we prove this in \cref{sec:auxiliary}. Then $S$ is updated to $S \cup \{\theta_S\}$, and we repeat this until no constraint is violated, \ie, $\min_{\theta_* \in \Theta} (\mathring{\pi}^{(i)}_S-\alpha\pi_0)^{\top}A^{\top}\theta_*\geq0.$

\section{Application I: Off-Policy Evaluation and Optimization}\label{sec:ope}
In this section, we apply our method to off-policy learning, where we use data collected by a logging policy to estimate the expected reward of another policy $\pi$ without ever deploying it. Previously, to ease exposition, we omitted dependence on contextual information in the definition of the reward function. In this section, we consider the more practical contextual setting \citep{LiCLW11,DudikELL14}.

\subsection{Tabular Off-Policy Evaluation and Optimization}
\label{sec:simplex_ope}

\noindent\textbf{Additional Notation and Setup:} Following \cref{sec:optimal_designs}, we consider the tabular action set. To model the contextual information, we assume that there exists a finite set of contexts $\cX$. A policy $\pi: \cX \to \Delta_{K - 1}$ is a mapping from a context to a probability distribution over actions, \ie, $\pi(a \mid x)$ is the probability of taking action $a \in \cA$ given context $x \in \cX$. We assume that the random reward for taking action $a$ under context $x$ is a $[0,1]$-valued random variable with mean $\bar{r}(x,a).$ We collectively denote $\bar{r}(x,\cdot)=(\bar{r}(x,1),\ldots,\bar{r}(x,K))^{\top}$ and $\bar{r}=(\bar{r}(x,\cdot))_{x\in\cX}.$ In what follows, we treat $\bar{r}$ as a $K \times \abs{\cX}$ matrix. We let $\cC$ be the distribution of the context. Let $V(\pi;\bar{r})=\sum_{x\in\cX} \cC(x)V(\pi(\cdot\mid x))$ and $V(\pi(\cdot\mid x);\bar{r})=\sum_{a\in\cA}\pi(a\mid x)\bar{r}(x,a)$ be the expected and conditional (on context $x$) expected reward, respectively, of policy $\pi.$ With some abuse of notation, we let $\pi(\cdot\mid x)= (\pi(1\mid x), \dots, \pi(K\mid x))\T$ be a vectorized policy $\pi$ conditioned on context $x$ and $g(\pie)=\max_{x\in\cX,a\in\cA}1/\pie(a\mid x).$

Our logging policy $\pie,$ whose expected reward is at least $\alpha$ of that of the production policy $\pi_0$, samples actions for $n$ times and collects a dataset $\cD = \set{(x_t, a_{t}, r_t)}_{t = 1}^n$ of size $n.$ Here $r_t\in[0,1]$ is a stochastic reward of action $a_t$ under context $x_t$ in round $t,$ with mean $\bar{r}(x_t,a_t).$

\noindent\textbf{Inverse Propensity Score (IPS) Estimator:} To estimate the expected reward of any policy $\pi$ from $\cD$, we use the asymptotically optimal and unbiased IPS estimator \citep{RosenbaumR83,WangAD17} as an example. The IPS estimator computes that value as
\begin{align}
	\hat{V}(\pi)
	= \frac{1}{n}\sum_{t = 1}^n \frac{\pi(a_t \mid x_t)}{\pie(a_t \mid x_t)} r_t\,.
	\label{eq:ips estimator}
\end{align}
Since $\pi(a_t\mid x_t),r_t\in[0,1],$ we know that each individual term in the IPS estimator is bounded in $[0,g(\pie)]$ and hence, $g(\pie)^2/4$-sub-Gaussian. 
Therefore, by Hoeffding's inequality (see \eg, equations (5.6) and (5.7) of \cite{LS18}), for any fixed policy $\pi$, $$|\hat{V}(\pi)-V(\pi)|\leq g(\pie)\sqrt{\frac{\log(2/\delta)}{2n}}$$ holds with probability at least $1-\delta.$ Intuitively, this means that we get a better estimator of $V(\pi)$ by minimizing $g(\pie).$  In what follows, we show how our prior results can help here. Specifically, we extend our results developed in \cref{sec:simplex} to the contextual setting to derive the logging policy $\pie$ that optimally minimizes $g(\pie).$ Then, we discuss how an optimized $g(\pie)$ can provide improved estimation guarantee for every possible policy through off-policy learning, and hence, benefit the downstream policy learning task. 

\noindent\textbf{No Side Information:} We begin by discussing how to optimize $g(\pie)$ when no side information about $\bar{r}$ is available. In this case, even if we have full access to the context distribution $\cC,$ we need to enforce $V(\pie(\cdot\mid x))\geq\alpha V(\pi_0(\cdot\mid x))$ across all $x\in\cX$ to ensure $V(\pie)\geq\alpha V(\pi_0),$ Otherwise, suppose that there exists $x\in\cX$ such that $V(\pie(\cdot\mid x))<\alpha V(\pi_0(\cdot\mid x)).$ Then one could set $\bar{r}(x',a)=0$ for all $x'\neq x$ to violate the safety constraint. In this case, we implement the water-filling method for each context $x\in\cX$ separately to minimize $\max_{a\in\cA}1/\pie(a\mid x)$, which subsequently minimizes $g(\pie)$ without violating the safety constraint.

\noindent\textbf{With Side Information:} In this case, we have access to side information $\bar{r}\in[L,U].$ To further incorporate the distribution of $\cX,$ we note that $V(\pie(\cdot \mid x))<\alpha V(\pi_0(\cdot\mid x))$ could possibly occur for some $x$ as long as $\pie$ performs better in other contexts. To this end, we formulate the optimization jointly over all $x\in\cX,$ \ie, 
\begin{align*}
	\max\ &\gamma\quad\\
	\mathrm{s.t.} \
	&\pie \geq \gamma\mathbf{1}_{K\times |\cX|}\,, \
	\pie(\cdot\mid x) \in \Delta_{K - 1}~\forall x\in\cX \\
	&\min_{\bar{r}\in[L,U]}\sum_{x\in\cX}\cC(x)(\pie(\cdot\mid x) - \alpha\pi_0(\cdot\mid x))\T \bar{r}(x,\cdot) \geq 0\,.
\end{align*} 
This optimization problem can be solved using the same duality trick as in \cref{sec:simplex_side}. We remark that if $L=\mathbf{0}_{K\times|\cX|}$ and $U=\mathbf{1}_{K\times|\cX|},$ this recovers the no side information case and we get the same solution as water-filling applied separately to each context.

\noindent\textbf{Performance Guarantee:} Recall that $g(\pie)$ is exactly the minimized objective in the above optimization problem. Now we are ready to show how safe optimal experimental design improves off-policy evaluation that further benefits the downstream optimization task.
\begin{theorem}
	\label{lemma:ips_error_bound} Let $\hat{V}(\pi)$ be the IPS estimate for the value of policy $\pi$ in \eqref{eq:ips estimator}. Then with probability at least $1-\delta,$ $$\max_{\pi} |\hat{V}(\pi) - V(\pi)|
	\leq 7g(\pie)\sqrt{\frac{|\cX|\log(4K|\cX|n/\delta)}{2n}}.$$ 
	Also let $\hat{\pi} = \argmax_\pi \hat{V}(\pi)$ and $\pi_* = \argmax_\pi V(\pi)$. Then
	$$V(\pi_*) - V(\pi)
	\leq 14g(\pie)\sqrt{\frac{|\cX|\log(4K|\cX|n/\delta)}{2n}}$$
	holds with probability at least $1 - \delta$.
\end{theorem} 
\begin{proof}[Proof Sketch]
	The complete proof is provided in \cref{sec:lemma:ips_error_bound}.
	
	Although it is straightforward to show that the confidence interval holds for a single fixed policy $\pi,$ directly applying the union bound over the entire policy space $\Pi$ would not lead to the desired statement because $\Pi$ contains infinitely many policies (see, \eg, corollary 11 of \cite{AgarwalK19}).

To overcome this challenge, we make use of the singleton/deterministic policies $\{\pi^{(k)}(\cdot\mid x)\}$ (\ie, $\pi^{(k)}(\cdot\mid x)$ assign probability 1 to action $k,$ and 0 to the rest) for each context, \ie,
\begin{align*}
    \pi^{(k)}(a\mid x)=\begin{cases}
  1  &  \text{when } a=k \\
  0 &  \text{otherwise.}
\end{cases}
\end{align*} 
We also extend the IPS estimator to the conditional expected reward, \ie,
\begin{align*}
    \hat{V}(\pi(\cdot\mid x))=\frac{1}{n_x}\sum_{t=1}^n\frac{\pi(a_t\mid x_t)\I{x_t=x}}{\pie(a_t\mid x_t)}r_t,
\end{align*}
where $n_x=\sum_{t=1}^n\I{x_t=x}$ is the number of times that the context $x$ is recorded. Through standard concentration inequality arguments (\ie, Hoeffding's inequality and union bound), we have with probability at least $1-\delta,$ for every possible $x\in\cX,$
\begin{align}\label{eq:simplex_ope_new6}
    \max_{k\in[K]}\left|\hat{V}(\pi^{(k)}(\cdot\mid x))-V(\pi^{(k)}(\cdot\mid x))\right|=O\left( g(\pie)\sqrt{\frac{\log(K|\cX|n/\delta)}{\max\{1,n_x\}}}\right).
\end{align}

To this end, we make two critical observations. First, any policy $\pi(\cdot\mid x),$ it can be expressed as convex combination of $\{\pi^{(k)}(\cdot\mid x)\}.$ Therefore, \eqref{eq:simplex_ope_new6} implies that
for every context $x$ and every policy $\pi(\cdot\mid x),$ with probability at least $1-\delta,$
\begin{align}\label{eq:simplex_ope_new7}
    \left|\hat{V}(\pi(\cdot\mid x))-  V(\pi(\cdot\mid x))\right|
    \leq& O\left(g(\pie)\sqrt{\frac{\log(K|\cX|n/\delta)}{\max\{1,n_x\}}}\right)\,.
\end{align}
Further, a policy $\pi$ can be viewed as a collection of policies for each individual context $x.$ Therefore, \eqref{eq:simplex_ope_new7} implies that with probability at least $1-\delta,$ 
\begin{align*}
     \max_{\pi}\left|\hat{V}(\pi)-V(\pi)\right|=O\left(g(\pie)\sqrt{\frac{|\cX|\log(K|\cX|n/\delta)}{n}}\right).
\end{align*}

The second claim follows directly from $\hat{V}(\pi)$ being close to $V(\pi)$ for any policy $\pi$ with a high probability.
\end{proof}
\begin{remark}[\textbf{Sharpness of the Bound}]
We note that according to corollary 11 of \cite{AgarwalK19}, it is shown that for a off-policy evaluation task with $N$ potential policies and $n$ pieces of logged data, the error bound would be of order $O(g(\pie)\sqrt{\log(N/\delta)/n}).$ Hence, it is easy to verify that our bound matches this up to logarithmic factors as $N=\Theta(K^{|\cX|})$ even if we only consider singleton/deterministic policies. 
\end{remark}

\subsection{Linear Off-Policy Evaluation and Optimization}
\label{sec:ope_linear}
Similarly to \cref{sec:simplex_ope}, we apply our results to contextual off-policy evaluation and optimization.

\noindent\textbf{Additional Notation:} We follow most of the notation in \cref{sec:simplex_ope} and recall that $\cA$ is the set of all actions, $a$ represents the individual elements in $\cA$ and $A$ is the matrix with each column being an element in $\cA.$ But for now, the reward parameter conditioned on context $x$ is $\theta_{*}(x)$ and $g(\pi)=\max_{x\in\cX,a\in\cA}
a\T G(\pi(\cdot\mid x))^{-1} a.$ The side information is defined as follows: for every $x\in\cX,$ $\theta_{*,x}\in\Theta_x=\{\theta\in\realset^d:(\theta-\bar{\theta}_x)^{\top}\bar{\Sigma}_x^{-1}(\theta-\bar{\theta}_x)\leq 1\}.$ We collectively denote $\theta_*=(\theta_{*,x})_{x\in\cX}$ and $\Theta=(\Theta_x)_{x\in\cX}.$ 

\noindent\textbf{Pseudo-Inverse (PI) Estimator:} To leverage the linear structure in the reward function, \citet{SwaminathanKAD17} proposed the PI estimator, which generalizes the IPS estimator, to estimate the expected reward of a policy $\pi.$ Specifically, let $G(\pi(\cdot\mid x))=\sum_{a\in\cA}\pi(a\mid x)aa^{\top},$ the PI estimator is 
\begin{align}
	\label{eq:pi_estimator} \hat{V}(\pi)=\frac{1}{n}\sum_{t=1}^n r_t\left(A\pi(\cdot\mid x_t)\right)^{\top}G(\pie(\cdot \mid x_t))^{-1}a_{t}\,,
\end{align} 
where $A \pi(\cdot\mid x)$ is the average action feature vector under $\pi(\cdot \mid x).$ Here we slightly overload our notation and use $G^{-1}$ as the pseudo-inverse of $G.$ \citet{SwaminathanKAD17} showed in Proposition 1 that $\hat{V}(\pi)$ is an unbiased estimator of $V(\pi)$. From the triangle and Cauchy-Schwarz inequalities, we have that
\begin{align}
	\nonumber\left|\left(A\pi(\cdot\mid x_t)\right)^{\top}G(\pie(\cdot \mid x_t))^{-1}a_t\right|=&\left|\sum_{a\in\cA}\pi(a\mid x_t)a^{\top}G(\pie(\cdot\mid x_t))^{-1}a_t\right|\\
	\nonumber\leq&\sum_{a\in\cA}\pi(a\mid x_t)\left|a^{\top}G(\pie(\cdot\mid x_t))^{-1}a_t\right|\\
	\nonumber\leq&\sum_{a\in\cA}\pi(a\mid x_t)\sqrt{a^{\top}G(\pie(\cdot\mid x_t))^{-1}a}\sqrt{a^{\top}_{t}G(\pie(\cdot\mid x_t))^{-1}a_t}\\
	\nonumber=&g(\pie)\,.
\end{align}
Therefore, each of the terms in the summand of \eqref{eq:pi_estimator} is $g(\pie)^2/4$-sub-Gaussian. 

\noindent\textbf{Safe Optimal Design with Side Information:} To incorporate the distribution of $\cX$ and the side information, we consider the following optimization problem
\begin{align*}
	\min\ g(\pie)\quad\st \ \pie(\cdot\mid x)\in\Delta_{K-1}~\forall x\in\cX\,,\min_{\theta_{*}\in\Theta}\sum_{x\in\cX}\cC(x)(\pie-\alpha\pi_0)^{\top}A^{\top}\theta_{*,x}\geq0\,.
\end{align*} 
This optimization problem can be solved analogously to that in \cref{sec:frank-wolfe}.

\noindent\textbf{Performance Guarantee:} As in the tabular case, $g(\pie)$ is exactly the minimized objective in the above optimization problem. We are now ready to link it to off-policy evaluation and optimization guarantees.

\begin{theorem}
	\label{lemma:pi_error_bound} Let $\lambda_{*}(x)$ be the minimum non-zero eigenvalue of $G(\pie(\cdot\mid x)),$ $\lambda_{*}=\min_{x\in\cX}\lambda_{*}(x),$ and $\hat{V}(\pi)$ be the PI estimate for the value of policy $\pi$ in \eqref{eq:pi_estimator}. Then with probability at least $1-\delta,$
	$$\max_{\pi}\left|\hat{V}(\pi)-V(\pi)\right|\leq3 g(\pie)\sqrt{\frac{d|\cX|\log(n/(\delta\min\{1,\sqrt{\lambda_*}\}))}{n}}.$$  Furthermore, let $\hat{\pi} = \argmax_\pi \hat{V}(\pi)$ and $\pi_* = \argmax_\pi V(\pi).$ Then with probability at least $1 - \delta,$
	$$V(\pi_*) - V(\pi)
	\leq 6g(\pie)\sqrt{\frac{d|\cX|\log(n/(\delta\min\{1,\sqrt{\lambda_*}\}))}{n}}.$$
\end{theorem} 
\begin{proof}[Proof Sketch]
	The complete proof is provided in \cref{sec:lemma:pi_error_bound}. We also comment on this bound in the upcoming \cref{remark:pi_error_bound}.
	
	Different than \ref{lemma:ips_error_bound}, we can no longer consider the singleton policies (as they might not even exist in $\cA$). To deal with the potentially large action set, we consider a discretization over the space of $\cS^{|\cX|}$ where $\cS=\{a\in\realset^d:\|a\|_2\leq1\}.$ Specifically, we let $\cQ$ be the $\min\{1,\sqrt{\lambda_*}\}/n$-cover of $\cS$ (\ie, for any $a\in\cS,$ there exists $a'\in\cQ$ such that $\|a-a'\|_2\leq \min\{1,\sqrt{\lambda_*}\}/n$), then we know that $\left|\cQ\right|\leq(3\min\{1,\sqrt{\lambda_*}\}/n)^d$ (which implies $\left|\cQ^{|\cX|}\right|\leq(3\min\{1,\sqrt{\lambda_*}\}/n)^{d|\cX|}$). To proceed, we define $A_{\cQ}$ to be the matrix that contains each element of $\cQ$ as its column. With slight abuse of notations, we define for the set of all possible deterministic policies that maps a context $x$ to an action in $\cQ$ as $\cH=\{h:\cX\to\cQ\}$, we have, from Hoeffding's inequality (see \eg, equations (5.6) and (5.7) of \cite{LS18}) and union bound, that with probability at least $1-\delta,$
	\begin{align}
		\label{eq:linear_ope2}\max_{h\in\cH}\left|\hat{V}(h)-V(h)\right|\leq g(\pie)\sqrt{\frac{d|\cX|\log(n/(\delta\min\{1,\sqrt{\lambda_*}\}))}{2n}}\,.
	\end{align}
	For any policy $\pi\in\Pi,$ we find the policy $h_{\pi}\in\cH$ such that for every $x\in\cX,$ $\|A\pi(\cdot\mid x)-A_{\cQ}h_{\pi}(\cdot|x)\|_2\leq \min\{1,\sqrt{\lambda_*}\}/n$ (This step is possible as $A\pi(\cdot\mid x)\in\cS$). Then we have
	\begin{align}
		\nonumber\left|\hat{V}(\pi)-V(\pi)\right|=&\left|\hat{V}(\pi)-\hat{V}(h_{\pi})+\hat{V}(h_{\pi})-V(h_{\pi})+V(h_{\pi})-V(\pi)\right|\\
		\label{eq:linear_ope3}\leq &\left|\hat{V}(\pi)-\hat{V}(h_{\pi})\right|+\left|\hat{V}(h_{\pi})-V(h_{\pi})\right|+\left|V(h_{\pi})-V(\pi)\right|\,,
	\end{align}
	where the second step follows from the triangle inequality. Now for the first term and third term, we can upper bound them due to the closeness between $\pi$ and $h_{\pi}$ (see \cref{sec:lemma:pi_error_bound} for more details), \ie,
	\begin{align}
		\nonumber\left|\hat{V}(\pi)-\hat{V}(h_{\pi})\right|\leq \frac{\sqrt{g(\pie)}}{n}\,, \quad \left|V(h_{\pi})-V(\pi)\right|\leq\frac{1}{n}\,.
	\end{align}
	The second term of \eqref{eq:linear_ope3} can be easily upper bounded as follows using \eqref{eq:linear_ope2}
	\begin{align*}
		\left|\hat{V}(h_{\pi})-V(h_{\pi})\right|\leq O\left( g(\pie)\sqrt{\frac{d|\cX|\log(n/(\delta\min\{1,\sqrt{\lambda_*}\}))}{2n}}\right)\,.
	\end{align*}
	
	Combining the above, we have
	\begin{align*}
		\max_{\pi\in\Pi}\left|\hat{V}(\pi)-V(\pi)\right|\leq O\left( g(\pie)\sqrt{\frac{d|\cX|\log(n/(\delta\min\{1,\sqrt{\lambda_*}\}))}{2n}}\right)\,.
	\end{align*}
	We note that the proof of the second claim is very similar to that of Lemma \ref{lemma:ips_error_bound} and is thus omitted.
\end{proof}
\begin{remark}[\textbf{Sharpness of the Bound}]\label{remark:pi_error_bound}
Compared to the performance guarantee in \cref{lemma:ips_error_bound}, the error bound in  \cref{lemma:pi_error_bound} is worse by a factor of $\tilde{O}(\sqrt{d}),$ this is somewhat expected as $\cA$ no longer permits simple basis (\ie, the singleton policy used in the proof of \cref{lemma:ips_error_bound}). Instead, the ``complexity" (under a suitable metric) of the $\cA$ is of order $O(2^d).$ Intuitively, plugging this into the result in \cref{lemma:ips_error_bound} would lead to a error bound of the same order as here.
\end{remark}

\subsection{Numerical Experiments}
\label{sec:experiments}

We conduct three experiments. In \cref{sec:illustrative example}, we illustrate the basic properties of our approach on a simple example. We evaluate it on a diverse set of problems in \cref{sec:synthetic problems} and on the MNIST dataset in \cref{sec:mnist dataset}. We note that all of these experiments are for the linear safe optimal design.

Our approach is implemented as described in \cref{sec:frank-wolfe} and we call it \safeod, which is an abbreviation for safe optimal design for ease of exposition. We compare it with two baselines. The first baseline is the G-optimal design $\pi_g$. The G-optimal design can be viewed as an unsafe variant of \safeod, \ie, $\alpha=0$ in \cref{eq:linear_side}. The second baseline is a mixture policy $\pi_\mathrm{mix} = \alpha \pi_0 + (1 - \alpha) \mathbf{1}_K / K$. This policy is guaranteed to satisfy the safety constraint but may not maximize information gain.

All logging policies $\pi$ are evaluated by three criteria:
\begin{itemize}
	\item The first is the \emph{design width} $\sqrt{g(\pi)}$, which is defined in \eqref{eq:g-optimal design} and reflects how well $\pi$ minimizes uncertainty over all actions. Lower values are better. 
	\item The second criterion is the \emph{safety violation} $\max_{\theta_* \in \Theta} (\alpha\pi_0 - \pi)\T A\T \theta_*$ (\cref{sec:general}), which measures how much $\pi$ violates the safety constraint for being close to the production policy $\pi_0$. Specifically, a positive value means that the safety constraint is violated; while a negative value means that the constraint is satisfied. Lower values implies a better safety performance. 
	\item The last metric is the \emph{off-policy gap}, which measures the suboptimality of the best off-policy estimated action on data collected by $\pi$. This metric is computed as follows. First, we drawn $\theta_* \in \Theta$, uniformly at random, and find the best action $a_*$ under $\theta_*$. Second, we collect a dataset $\cD$ of size $n = 10 d$, were the noisy observation of action $a$ is $a\T \theta_* + \varepsilon$ for $\varepsilon \sim \cN(0, 1)$. Finally, we compute the MLE of $\theta_*$ from $\cD$, which we denote by $\hat{\theta}$, and find the best action $\hat{a}$ under $\hat{\theta}$. The off-policy gap is $(a_* - \hat{a})\T \theta_*$ and we estimate it from $1\,000$ random runs for any given logging policy, as described above.
\end{itemize} 
\subsubsection{Illustrative Example}
\label{sec:illustrative example}

In the first example, $\cA = \set{(1, 0), (0, 1)}$, $\pi_0 = (0.2, 0.8)$, and $\alpha = 0.9$; and $\Theta$ is given by $\bar{\theta} = (1, 2)$ and $\bar{\Sigma} = 0.1 I_2$. In this case, for any $\theta_* \in \Theta$, $\pi_0$ takes the most rewarding action with a high probability of $0.8$. Therefore, \safeod cannot differ much from $\pi_0$ and is ${\pie} = (0.330, 0.670)$. This design satisfies the safety constraint and its width is $1.74$. In comparison, the G-optimal design is $\pi_g = (0.5, 0.5)$ and obviously violates the safety constraint. For instance, even at $\bar{\theta}$, the constraint violation is $0.9 \cdot (0.2 \cdot 1 + 0.8 \cdot 2) - 0.5 \cdot 3 = 0.12$. However, its width is only $1.414$. The mixture policy $\pi_\mathrm{mix}$ satisfies the safety constraint but its width is $2.085$, about $15\%$ higher than in \safeod.

In the second example, we set $\bar{\theta} = (2, 1)$. In this case, for any $\theta_* \in \Theta$, $\pi_0$ takes the least rewarding action with a high probability of $0.8$. Thus \safeod can depart significantly from $\pi_0$ and is $\pie = (0.5, 0.5)$. This design satisfies the safety constraint and its width is $1.141$. The G-optimal design coincides with \safeod, \ie, $\pi_g = {\pie}$. The mixture policy $\pi_\mathrm{mix}$ also satisfies the safety constraint but its width is $2.085$.

In summary, \safeod combines the best properties of $\pi_g$ and $\pi_\textrm{mix}$. When the safety constraint is strict, \safeod satisfies it. When it is not, \safeod has a low width, similarly to the G-optimal design.

\subsubsection{Synthetic Problems}
\label{sec:synthetic problems}

We also experiment with the following randomly generated problems. The number of actions is $K = 100$ and their feature vectors are drawn uniformly from a $d$-dimensional unit sphere. The production policy $\pi_0$ is drawn uniformly from a $(K - 1)$-dimensional simplex. The set $\Theta$ is defined by $\bar{\Sigma} = I_d$ and $\bar{\theta}$, where the latter is drawn uniformly from a $d$-dimensional hypercube $[1, 2]^d$. We vary $d$ and $\alpha$, and have $50$ independent experiments for each setting.

In \cref{fig:d=4 alpha=0.900}, we report results for $d = 4$ and $\alpha = 0.9$. We observe that the G-optimal designs have low widths but also violate the safety constraint. On the other hand, the mixture policy always satisfies the safety constraint but has high design widths. \safeod strikes the balance between the two objectives, by minimizing the design width under the safety constraint. In all cases, design widths correlate with off-policy gaps, which means that the optimized objective translates to improvements in off-policy optimization.
\begin{figure}[!ht]
	\vspace{-0.1in}
	\centering
	\includegraphics[width=5.6in]{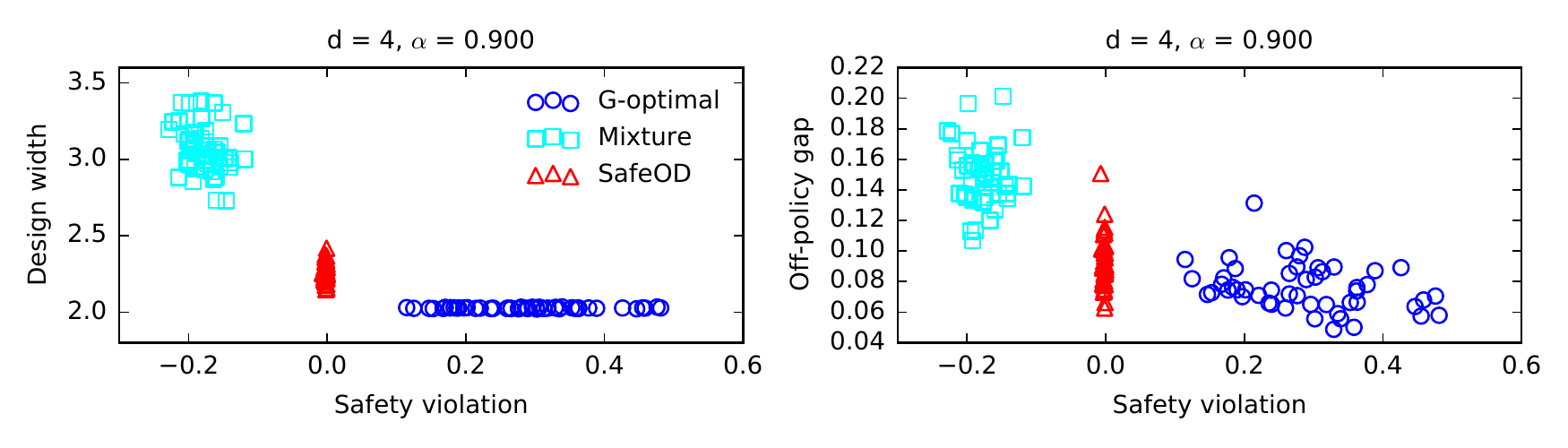} \\
	\vspace{-0.2in}
	\caption{Comparison of \safeod to the G-optimal optimal design and the mixture policy in $50$ random runs.}
	\label{fig:d=4 alpha=0.900}
\end{figure}

In \cref{fig:d}, we fix $\alpha = 0.9$ and vary $d$; while in \cref{fig:alpha}, we fix $d = 4$ and vary $\alpha$. In general, we observe that \safeod performs similarly to the G-optimal design whenever it is easy to satisfy the safety constraint, when the number of features $d$ is large or the safety parameter $\alpha$ is small. In all other cases, \safeod produces designs of higher widths and off-policy gaps in exchange for satisfying the safety constraint. The mixture policy always satisfies the safety constraint but has high design widths and off-policy gaps.
\begin{figure}[!ht]
	\vspace{-0.1in}
	\centering
	\includegraphics[width=5.6in]{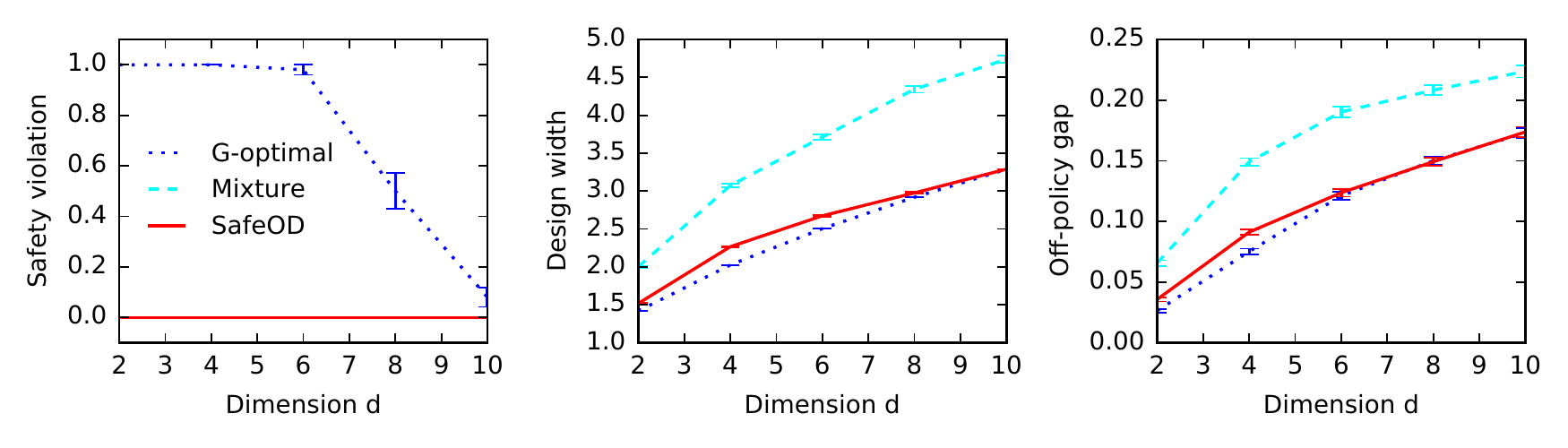} \\
	\vspace{-0.2in}
	\caption{Comparison of \safeod to the G-optimal optimal design and the mixture policy. We fix the safety parameter at $\alpha = 0.9$ and vary $d$. The safety violation is the fraction of violated safety constraints in $50$ runs.}
	\label{fig:d}
\end{figure}
\begin{figure}[!t]
	\vspace{-0.1in}
	\centering
	\includegraphics[width=5.6in]{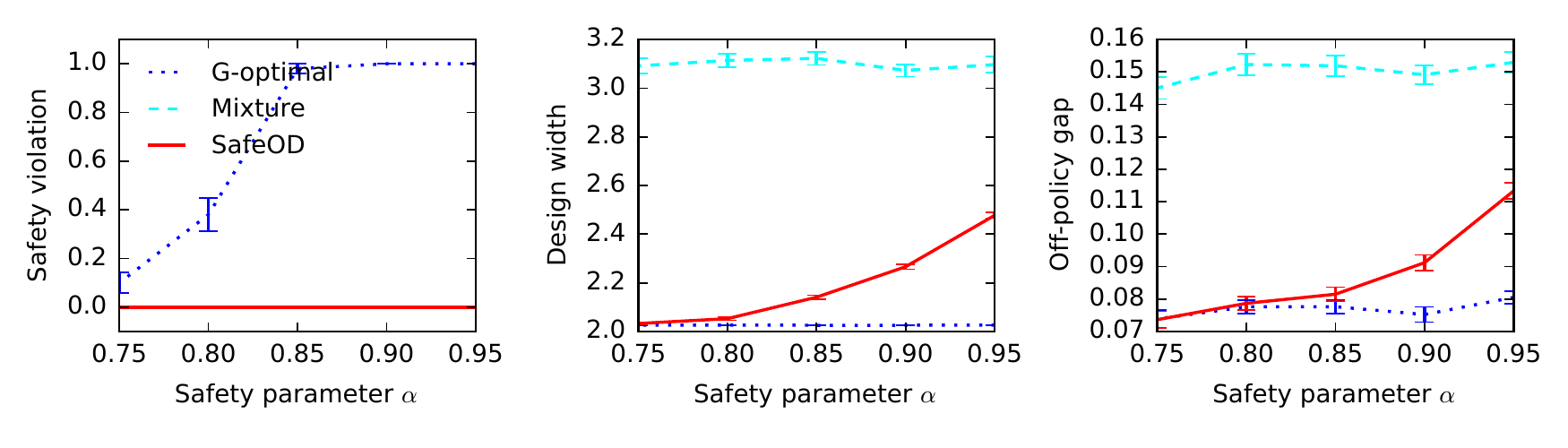} \\
	\vspace{-0.25in}
	\caption{Comparison of \safeod to the G-optimal optimal design and the mixture policy. We fix $d = 4$ and vary the safety parameter $\alpha$. The safety violation is the fraction of violated safety constraints in $50$ runs.}
	\label{fig:alpha}
\end{figure}
\subsubsection{MNIST Dataset}
\label{sec:mnist dataset}

The last experiment is conducted on the MNIST dataset \citep{lecun10mnist}. This experiment is a more realist variant of that in \cref{sec:synthetic problems},  where the actions $\cA$ and the safety ellipsoid $\Theta$ are estimated from a real-world dataset.

\begin{figure*}[t]
	\centering
	\includegraphics[width=5.6in]{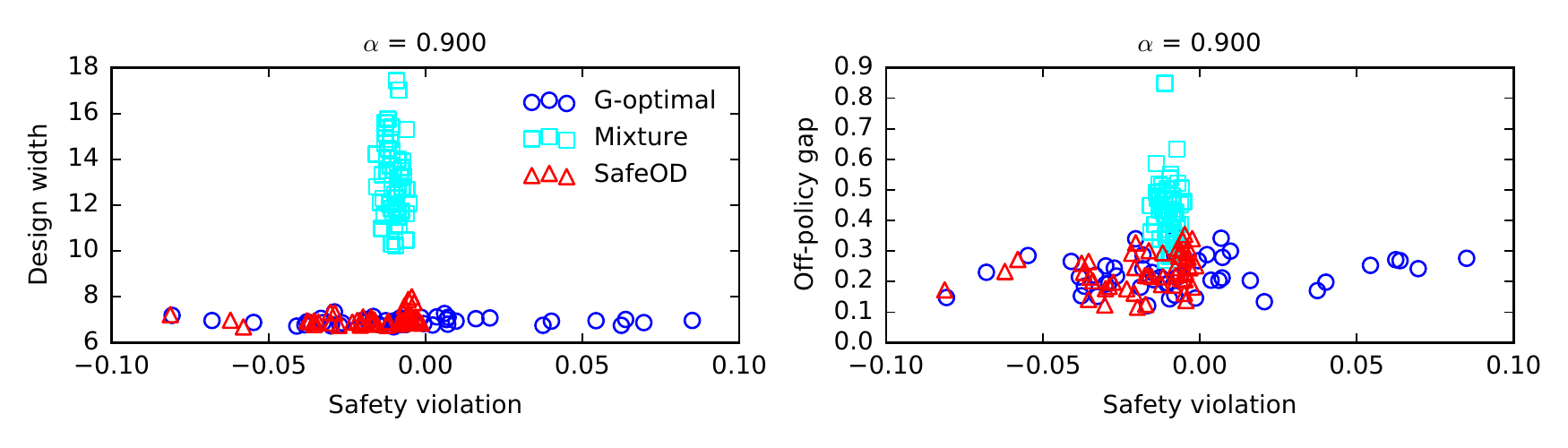}
	\caption{Comparison of \safeod to the G-optimal optimal design and the mixture policy on the MNIST dataset.}
	\label{fig:mnist alpha=0.900}
\end{figure*}

This experiment is conducted as follows. First, we choose a random digit and assign it reward one. All remaining digits have reward zero. Second, we take flattened images of digits as their feature vectors (see, \eg, \cite{Olah14}) and learn a least-squares regressor on these data. We set $\bar{\theta}$ and $\bar{\Sigma}$ to its weights and the inverse covariance matrix, respectively. These two quantities define $\Theta$, an ellipsoid that is likely to contain the optimal unknown weights $\theta_*$. Finally, we choose feature vectors of $K = 100$ random digit images as the action set $\cA$. The production policy $\pi_0$ is drawn uniformly from a $(K - 1)$-dimensional simplex. This is repeated $50$ times and our results are reported in \cref{fig:mnist alpha=0.900}.

\begin{figure*}[!t]
	\centering
	\includegraphics[width=5.6in]{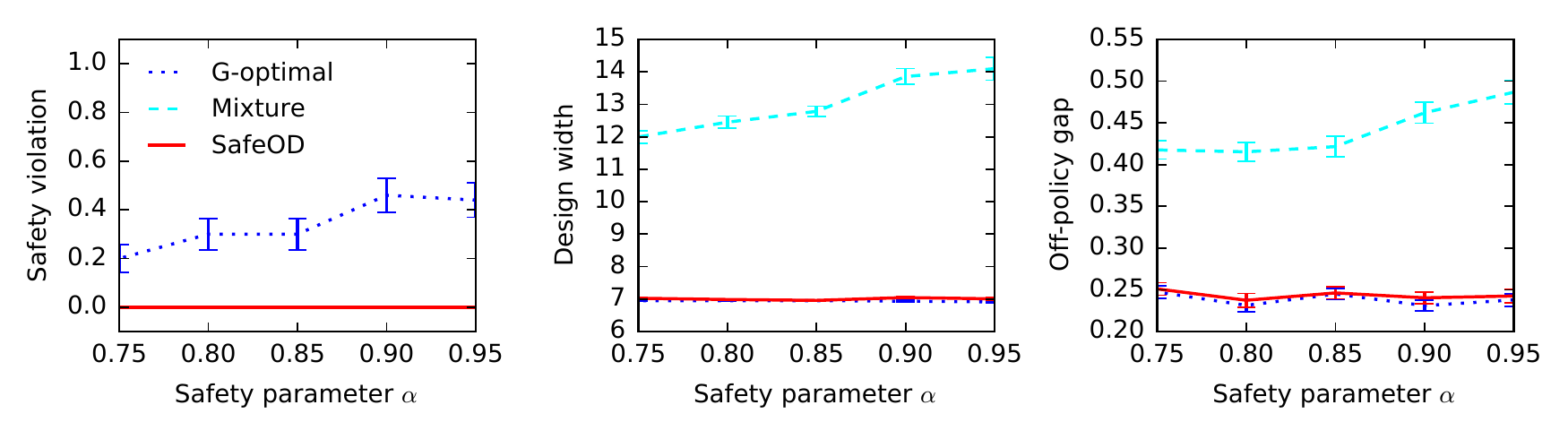}
	\caption{Comparison of \safeod to the G-optimal optimal design and the mixture policy on the MNIST dataset. We vary the safety parameter $\alpha$. The safety violation is the fraction of violated safety constraints in $50$ runs.}
	\label{fig:mnist alpha}
\end{figure*}

We observe similar trends to \cref{fig:d=4 alpha=0.900}. The G-optimal designs have slightly lower widths and better off-policy performance than \safeod, but can significantly violate the safety constraint. The comparable widths and off-policy performance of \safeod indicate that the cost of satisfying the safety constraint in real-world data may be low. The mixture policy always satisfies the safety constraint but has the highest design widths and worst off-policy performance.

As in \cref{fig:alpha}, we vary the safety parameter $\alpha$ in \cref{fig:mnist alpha}. For all $\alpha$, \safeod performs similarly to the G-optimal design but never violates the safety constraint. Again, the mixture policy always satisfies the safety constraint but has high design widths and off-policy gaps. This experiment shows that all observed trends in \cref{fig:mnist alpha=0.900} generalize beyond $\alpha = 0.9$.

\section{Application II: Safe Online Learning with Low Adaptivity}
\label{sec:explore}
To further demonstrate the power of our results, we generalize our method and apply it adaptively to the problem of safe/conservative exploration in multi-armed bandits \citep{WuSLS16}. In conservative exploration, we iteratively select actions with initially unknown random reward, and learn to maximize our cumulative reward while ensuring that our expected cumulative reward is always at least $\alpha$ of the expected cumulative reward obtained by a default action. 

\noindent\textbf{Safe Online Learning Setup:} We follow the setting in \cite{WuSLS16} and most of the notation in \cref{sec:optimal_designs}. We assume that the interaction lasts for a total of $n$ rounds and there is an additional default action $0$ whose expected reward $\bar{r}(0)\in[0,1]$ is known in advance (see \cref{remark:default_known} for this assumption). For every round $t=1,\ldots,T,$ selecting an action $a_t\in\{0\}\cup[K]$ would generate a random reward $r_t\in[0,1]$ with mean $\bar{r}(a_t).$ Our goal is to maximize the expected total reward $\E[\sum_{t=1}^T\bar{r}(a_t)]$ by following some (possibly randomized) policy $\pi.$ Here, in round $t,$ $\pi$ takes all historical observations up to round $t$ as inputs, and outputs the action to be chosen in round $t.$ Throughout, the safety constraint mandates that our expected cumulative reward is always at least $\alpha$ of the expected cumulative reward obtained by action $0$ with high probability (we defer a discussion on this constraint to the forthcoming \cref{remark:safety}), \ie, for some positive number $\delta <1$ (an input parameter),
\begin{align}\label{eq:safety}
	\Pr\left(\E\left[\sum_{s=1}^t\bar{r}(a_s)\right] \geq \alpha t\bar{r}(0)~\forall t\in[T]\right)\geq1-\delta.
\end{align}
We use the notion of regret, which is the difference between the maximum expected total reward and our expected total reward, to measure our performance, \ie, let $a^{opt} = \argmax_{a\in[K]\cup\{0\}}\bar{r}(a),$ 
\begin{align*}
	\regret(\pi) = T\bar{r}(a^{opt}) - \E\left[\sum_{t=1}^T\bar{r}(a_t)\right].
\end{align*}

\begin{remark}[\textbf{Connection with Safe Exploration Setting in \cref{sec:optimal_designs}}]
	We point out that in the safe online learning setting, one can view the production policy $\pi_0$ as the one that always selects action 0 and our goal is to solve the regret minimization task while respecting the safety constraint over time.
\end{remark}
\begin{remark}[\textbf{Known Expected Reward of Default Action}]\label{remark:default_known}
    It is worth noting that prior works also consider the setting where $\bar{r}(0)$ can be unknown, but for ease of exposition, we assume that the value of $\bar{r}(0)$ is given. We leave the case of unknown $\bar{r}(0)$ as future work.
\end{remark}

\subsection{Safe Phased-Elimination}
In this section, we apply the results developed in \cref{sec:simplex} adaptively to propose a Safe Phased-Elimination (\spe) algorithm for the safe online learning problem. We defer the discussion on how this algorithm is related, but critically different from the classic phased-elimination algorithm (see \eg, chapter 22 of \cite{LS18} or \cite{AO10}), to the forthcoming \cref{sec:intuition}. 

As its name suggested, \spe~runs in phases and maintains the set of plausible actions dynamically. At the beginning of each phase, \spe~first computes a safe optimal design w.r.t. the set of plausible actions. During this, it uses the data collected from the last phase as side information. Next it selects each action according to the design. At the end of a phase, \spe~utilizes the data collected during this phase to compute each action's estimated mean reward. Actions that are likely to be sub-optimal (\ie, with an estimated mean reward significantly smaller than the largest estimated mean reward) are then eliminated. 

\noindent\textbf{Additional Notations for the Algorithm:} For each phase $h=1,2,\ldots,$ we use $\cA_h$ to denote the set of remaining actions up to phase $h$, $\epsilon_h=2^{-h}$, and $g_h(\pi) = \max_{a\in\cA_h\setminus\{0\}}\pi^{-1}(a).$ Here, $\cA_h$ tracks the set of plausible actions and $\cA_1$ is initialized to $\cA\cup\{0\},$ $\epsilon_h$ is the targeted width of the (high probability) confidence interval for each action after phase $h;$ and similar to the role of $g(\pi)$ in \cref{sec:optimal_designs} and \cref{lemma:ips_error_bound}, $g_h(\pi)$ controls the variance proxy of the data samples collected by $\pi$ in phase $h.$ Intuitively, in order to achieve the targeted width of the confidence interval $\epsilon_h$, the length of phase $h$ would scale with $g_h(\pi)$ and $\epsilon_h^{-2}$ if $\pi$ is implemented (interested readers are referred to chapter 21.1 of \cite{LS18} for an detailed explanation). 

We also use $t_h$ and $T_h$ to denote the first round and the total number of rounds in phase $h.$ For each action, we define \begin{align}\label{eq:empirical}
   N_h(a) = \sum_{t=t_h}^{t_h+T_h-1}\I{a_t=a},\quad\hat{r}_h(a)= \frac{\sum_{t=t_h}^{t_h+T_h-1}\I{a_t=a}r_t}{N_h(a)}, 
\end{align} 
as the number of times that action $a$ is selected during phase $h$ and the corresponding estimated mean reward via the direct method. Since the expected reward of the default action 0 is known, we denote by $\hat{r}_h(0)=\bar{r}(0)$ the mean reward of the default action 0. The side information (\ie, high probability confidence intervals derived from data collected in the previous phase) for each action $a\in\cA_h\setminus\{0\}$ is then defined as
\begin{align}\label{eq:ci}
	L_h(a)=\hat{r}_{h-1}(a)-\sqrt{\frac{\log(KT^4/\delta)}{N_{h-1}(a)}}, \quad 	U_h(a)=\hat{r}_{h-1}(a)+\sqrt{\frac{\log(KT^4/\delta)}{N_{h-1}(a)}}.
\end{align} 
With some abuse of notations, we define $L_h(0)=0,U_h(0)=1$ for all $h.$ We remark that these are in fact nothing but just the upper and lower confidence bounds constructed, using data collected in phase $h-1,$ at the beginning of phase $h.$ But instead of using them for the upper confidence bound-type algorithms, we make use of them as side information in optimal design.
\begin{algorithm}[htbp]
\SingleSpacedXI
  \caption{Safe Phased-Elimination (\spe) Algorithm} \label{alg:spe}
  \begin{algorithmic}[]
    \State \textbf{Input:} $\cA\,,\ \bar{r}(0)\,,$ and $\delta$
    \State $\cA_1\leftarrow\cA\cup\{0\}\,,\ t\leftarrow 1\,,\ L_1\leftarrow \bm{0}_{|\cA_1|}\,,\ U_1\leftarrow \bm{1}_{|\cA_1|}\,$
\For {$h=1,2,\ldots$}
    \State{$t_h\leftarrow t\,,\ N_h\leftarrow\bm{0}_{|\cA_h|}\,,\ \epsilon_h\leftarrow 2^{-h}\,,$ compute $\pi_h$ according to \eqref{eq:tabular_spe_safety}}\qquad  \texttt{\color{blue}// Step 1}
    \For {$a\in \cA_h\cup\{0\}$ (in ascending order of $a$)}
    \While{$N_h(a)<0.5\pi_h(a)g_h(\pi_h)\epsilon_h^{-2}\log(KT^4/\delta)$ \textbf{and} $t<T$} \qquad  \texttt{\color{blue}// Step 2} 
    \State{$a_t\leftarrow a\,,$ receive $r_t\,,\ t\leftarrow t+1$}
    \EndWhile
    \EndFor
    \If{$t\geq T$}
    \State {\textbf{Break}}
    \Else
    \State{$\hat{r}_h(a)\leftarrow\frac{\sum_{t=t_h}^t\I{a_t=a}r_t}{N_h(a)}\,,L_{h+1}(a)\leftarrow\hat{r}_h(a)-\sqrt{\frac{\log(KT^4/\delta)}{N_{h}(a)}}\,,\ U_{h+1}(a)\leftarrow\hat{r}_h(a)+\sqrt{\frac{\log(KT^4/\delta)}{N_{h}(a)}}\ \forall a\in\cA_h\setminus\{0\}\,,\ \cA_{h+1}\leftarrow\cA_h$}\qquad  \texttt{\color{blue}// Step 3}
    \For{$a\in\cA_h\setminus\{0\}$}
    \If {$\hat{r}_h(a)\leq \max_{a\in\cA_h}\hat{r}_h(a)-2\epsilon_h$}
    \State{$\cA_{h+1}\leftarrow\cA_{h+1}\setminus\{a\}$}
    \EndIf
    \EndFor
    \EndIf
\EndFor
\end{algorithmic}
\end{algorithm}

\spe~proceeds according to the following steps (its formal description is provided in \cref{alg:spe}): 

\noindent\textbf{Step 1. Compute Safe Optimal Design:} It first computes the safe optimal design $\pi_h$ with side information as follows:
	\begin{align}
		\min \ &g_h(\pi_h) \nonumber \\
		\mathrm{s.t.} \
		& \pi_h\in\Delta_{|\cA_h|-1}\,, \nonumber \\ 
		&\min_{r\in[L_h,U_h]}\sum_{a\in\cA_h\setminus\{0\}}\pi_h(a)r(a)+\pi_h(0)\bar{r}(0)
		\geq\alpha\bar{r}(0) \,.\label{eq:tabular_spe_safety}
	\end{align} 
	For ease of analysis, if there are multiple optimal solutions, we assume that $\pi_h$ would be the one that maximizes $\pi_h(0).$ We remark that the optimization setup in \eqref{eq:tabular_spe_safety} and the one in \eqref{eq:tabular_no_side} are closely related. The critical difference is that the objective function is defined w.r.t. the plausible actions (\ie, $\cA_h\setminus\{0\}$). 

\noindent\textbf{Step 2. Batched Exploration:} 
	Same as other phased-elimination algorithms (see, \eg, chapter 22 of \cite{LS18}), \spe~then chooses each action $a$ (in ascending order of $a$ to ensure safety constraint is met) for $0.5\pi_h(a)g_h(\pi_h)\epsilon_h^{-2}\log(KT^4/\delta)$ times (with proper rounding if necessary, we also note that this is $N_h(a)$) in a batched manner (\ie, only $|\cA_h|$ number of action switches would be needed). The algorithm may end here if it reaches round $T$.

\noindent \textbf{Step 3. Action Elimination:} For each action $a\in\cA_h\setminus\{0\},$ \spe~computes the estimated mean reward $\hat{r}_h(a).$ It then eliminates actions that are sub-optimal with a high probability. That is, it identifies the action with highest estimated mean reward, \ie, $\argmax_{a\in\cA_h}\hat{r}_h(a),$ and eliminates any action $a\in|\cA_h|\setminus\{0\}$ whose estimated mean reward is at least $2\epsilon_h$ less than the highest estimated mean reward, \ie, 
\begin{align}\label{eq:elimin1}
	\hat{r}_h(a)\leq \max_{a\in\cA_h}\hat{r}_h(a)-2\epsilon_h.
\end{align}

\subsection{Theoretical Analysis}
We first show that the number of policy updates, \ie, $\sum_{t=1}^{T}\I{\pi_t\neq \pi_{t-1}}$, for \spe~is small. 
\begin{lemma}
	The number of policy updates of \spe~is of order $O(K\log T).$
\end{lemma}
\begin{proof}
	For each phase $h,$ the number of action switches is at most $(K+1)$ by virtue of the batched exploration in step 2. We also note that for each phase $h,$ it lasts for $\sum_{a\in\cA_h}N_h(a)$ rounds. By definition of $N_h(a)$ and the fact that $g_h(\pi_h)\geq 1,$
	\begin{align*}
		\sum_{a\in\cA_h}N_h(a)=0.5g_h(\pi_h)\epsilon^{-2}_h\log(KT^4/\delta)\geq 2^{h-1}.
	\end{align*}
Therefore, the \spe~algorithm has at most $\log_2(2T)$ phases. The statement then follows.
\end{proof}
\begin{remark}
	Compared to existing works in safe online learning \citep{WuSLS16,NIPS2017_bdc4626a,pmlr-v115-li20b}, which would potentially require $\Theta(T)$ number of policy updates, the \spe~algorithm reduces the number of policy updates exponentially. This is particularly beneficial during deployment as now the policy is quasi-fixed.
\end{remark}

We are now ready to provide the regret analysis for the \spe~algorithm. We show that the algorithm satisfies the safety constraint \eqref{eq:safety} and establish its regret upper bound.
\begin{theorem}\label{thm:safety_regret}
	For any given $\delta<1,$ if we follow the \spe~algorithm, we have  that our expected cumulative reward is always at least $\alpha$ of the expected cumulative reward obtained by action $0$ with probability at least $1-\delta,$ \ie, $$\Pr\left(\E\left[\sum_{s=1}^t\bar{r}(a_s)\right] \geq \alpha t\bar{r}(0)~\forall t\in[T]\right)\geq1-\delta.$$ Further, the regret of the \spe~algorithm satisfies
	\begin{align*}
		\regret(\spe\text{ algorithm})=O\left(\sqrt{KT}\log(KT/\delta)+\frac{K\log(KT/\delta)}{(1-\alpha)\bar{r}(0)}\right).
	\end{align*}
\end{theorem} 
Before presenting the proof, a couple of remarks for this result are in order.
\begin{remark}[\textbf{Tightness of the Regret Bound}]
    Compared to the lower bound for this setting developed in theorem 9 of \cite{WuSLS16}, our regret bound is optimal up to logarithmic factors and it is of the same order as the algorithm proposed in \cite{WuSLS16} (see theorem 2 therein). Intuitively, larger $\alpha$ and/or smaller $\bar{r}(0)$ \ie, more restrictive safety constraint and/or worse expected reward of the default action $0,$ would lead to worse regret bound. Both of them are due to the fact that we need to respect the safety constraint, and hence would need to select action 0, especially when $t$ is small and we do not have accurate estimates of the other actions' expected reward.
\end{remark} 
\begin{remark}[\textbf{Probabilistic Safety Constraint}]\label{remark:safety}
    We note that the safety constraint \eqref{eq:safety} is required to be satisfied with probability at least $1-\delta,$ where $\delta$ is an input parameter. This is the same as the setting in \cite{WuSLS16}. For \spe, the reason that this constraint does not hold with probability 1 is because we use the high probability confidence intervals $L_h$ and $U_h$ as side information in Step 1. 
\end{remark}

\subsection{Proof of \cref{thm:safety_regret}}
The proof consists of two parts. We first show the safety constraint is met and then establish the regret upper bound. 

\noindent\textbf{Safety Constraint:} First of all, since action 0 is never removed, the safe optimal design problem \eqref{eq:tabular_spe_safety} in each phase has an non-empty feasible region. The following lemma shows that $L_h$ and $U_h$ are high probability lower and upper bounds for $\bar{r}.$
\begin{lemma}\label{lemma:spe0}
For every phase $h$ and every action $a\in\cA_h,$ we have $\Pr(\bar{r}(a)\in[L_h(a),U_h(a)])>1-\delta.$ Moreover, $[L_h(a),U_h(a)]\subseteq[\hat{r}_h(a)-\epsilon_h,\hat{r}_h(a)+\epsilon_h]$ for all $h$ and $a\in\cA_h.$
\end{lemma}
The proof of this lemma is provided in \cref{sec:lemma:spe0}. By virtue of the safe optimal design and the way we select our actions, \ie, in the ascending order of $a$ (so that action 0 is chosen first), one can immediately verify that the \spe~algorithm satisfies the safety constraint.

\noindent\textbf{Regret Upper Bound:} For the regret upper bound, we define the event $\cE:$
\begin{align*}
	\cE=\{\bar{r}(a)\in[L_h(a),U_h(a)]\subseteq[\hat{r}_h(a)-\epsilon_h,\hat{r}_h(a)+\epsilon_h]~\forall h~\forall a\in\cA_h\}.
\end{align*}
By \cref{lemma:spe0}, we know that $\Pr(\cE)\geq 1-\delta/T.$ To this end, we define $\Delta(a) = \bar{r}(a^{opt}) - \bar{r}(a)$
as the sub-optimality gap of each action $a.$ Then, we show that on event $\cE,$ $a^{opt}$ is never eliminated, and whenever $\epsilon_h=o(\Delta(a)),$ action $a$ would be removed due to the elimination criterion \eqref{eq:elimin1}.
\begin{lemma}\label{lemma:spe2}
On $\cE,$ $a^{opt}$ is never eliminated. Moreover, for any phase $h,$ if an action $a$ belongs to $\cA_h,$ it must be that $\Delta(a)<5\epsilon_h.$ 
\end{lemma}
The proof of this lemma is provided in \cref{sec:lemma:spe2}. Note that if the default action 0 is sub-optimal, we would need to account for regret incurred by it, we thus distinguish two different cases:
\begin{itemize}
\item \textbf{Case 1. Default Action 0 is Optimal:} On event $\cE$, we can upper bound the regret of \spe~as follows.
\begin{lemma}\label{lemma:spe6}
On $\cE,$ when action 0 is optimal, $$\regret(\spe)=O\left(\sqrt{KT}+\max_h\left(1-\pi_h(0)-\sum_{a:L_h(a)\geq\bar{r}(0)}\pi_h(a)\right)g_h(\pi_h)\sqrt{\frac{T}{K}}\log(KT/\delta)\right).$$
\end{lemma}
The proof of this lemma is provided in \cref{sec:lemma:spe6}. To proceed, we provide an upper bound for the quantity $(1-\pi_h(0)-\sum_{a:L_h(a)\geq\bar{r}(0)}\pi_h(a))g_h(\pi_h).$
\begin{lemma}\label{lemma:spe5}
On $\cE,$ we have $(1-\pi_h(0)-\sum_{a:L_h(a)\geq\bar{r}(0)}\pi_h(a))g_h(\pi_h)<K$ for all $h$ regardless of the optimality of action 0. 
\end{lemma}
The proof of this lemma is provided in \cref{sec:lemma:spe5}. Combining the above, we have $\regret(\spe)=O(\sqrt{KT}\log(KT/\delta)).$

\item \textbf{Case 2. Default Action 0 is Sub-Optimal:} In this case, we need to additionally consider the regret incurred by selecting the default action 0. Since action 0 is never removed from the action set to ensure safety, we turn our attention to the side information, which can help to gradually rule out the use of action 0, \ie, when $L_h(a)\geq \bar{r}(0)$ for all $a\in\cA_h\setminus\{0\}.$

 \begin{lemma}\label{lemma:spe3}
On $\cE,$ if action 0 is sub-optimal, $\pi_h(0)=0$ for any $h\geq h_0= 1+\log_2(4/(\Delta(0)+(1-\alpha)\bar{r}(0))).$
\end{lemma}
The proof of this lemma is provided in \cref{sec:lemma:spe3}. Now, we make the observation that for a phase $h<h_0,$ if $L_h(a)\geq\bar{r}(0),$ then 
\begin{align}
    \Delta(a)=\bar{r}(a^{opt})-\bar{r}(a)\leq\bar{r}(a^{opt})-L_h(a)\leq \bar{r}(a^{opt})-\bar{r}(0)\leq \Delta(0),
\end{align} 
where the first step is by definition of $\Delta(a)$ and the second step is by event $\cE.$ Denoting $\cA_*$ as the set of actions $a$ such that $L_h(a)\geq\bar{r}(0)$ for some $h,$ the regret incurred by action 0 and those in $\cA_*$ during the first $h_0$ phases is upper bounded as $\sum_{h=1}^{h_0}\Delta(0)(\sum_{a\in\cA_*\cup\{0\}}N_h(a)).$ We provide an upper bound for this quantity in the following lemma.
\begin{lemma}\label{lemma:spe7}
On $\cE,$ if action 0 is sub-optimal, $$\sum_{h=1}^{h_0}\Delta(0)\left(\sum_{a\in\cA_*\cup\{0\}}N_h(a)\right)=O(\max_h\left(g_h(\pi_h)\epsilon^{-1}_h\right)\log(KT/\delta)).$$
\end{lemma}
The proof of this lemma is provided in \cref{sec:lemma:spe7}.
To proceed, we provide an upper bound for $g_h(\pi_h)\epsilon^{-1}_h.$
\begin{lemma}\label{lemma:spe4}
On $\cE,$ we have $g_h(\pi_h)\epsilon^{-1}_h\leq8K(1-\alpha)^{-1}\bar{r}(0))^{-1}$ for any $h< h_0.$
\end{lemma}
The proof of this lemma is provided in \cref{sec:lemma:spe4}. Combining the above, we have that the regret incurred by action 0 and those in $\cA_*$ during the first $h_0$ phases is $O(K\log(KT/\delta)(1-\alpha)^{-1}\bar{r}(0))^{-1}).$ For the other actions, one can similarly upper bound their regret as case 1. Hence, the regret of \spe~is at most 
\begin{align*}
		\regret(\spe)=O\left(\sqrt{KT}\log(KT/\delta)+\frac{K\log(KT/\delta)}{(1-\alpha)\bar{r}(0)}\right).
\end{align*}
\end{itemize}
The conclusion follows by combining the above two cases and noticing that $\cE$ holds with probability at least $1-\delta/T$.

\subsection{Discussions and Comparisons with Classic Phased-Elimination}\label{sec:intuition} 

In this section, we present the rationale for our design, and highlight its similarities and differences from the existing phased-elimination algorithm.

As in the classic phased-elimination algorithm (see \eg, chapter 22 of \cite{LS18} or \cite{AO10}), \spe~iterates in phases of increasing length. In each phase $h,$ the algorithm collects sample for all remaining actions and eliminates an action if its estimated reward is $2\epsilon_h$ less than the largest estimated reward using $\tilde{\Theta}(\epsilon_h^{-2})$ samples. By doing so, the algorithm ensures that the optimal action is not eliminated with a high probability while the sub-optimal actions are eliminated once detected. 

However, \spe~is critically different than classic phased elimination algorithm in order to ensure the safety constraint is met without too much extra cost. Specifically:
	\begin{enumerate}
		\item In Step 1, \spe~computes the safe optimal design based on data collected in the previous phase. Then, in Step 2, we explore the plausible actions according to the output of Step 1. We emphasize that in this step, action 0 is selected at the beginning of each phase (if $\pi_h(a)>0$). By doing so, we can simultaneously ensure the safety constraint is met and the information gain is maximized. 
		\item In Step 3, we point out that the default action 0 is never eliminated. At first glance, this may result in excessive use of action 0 even if it is sub-optimal. In fact, this would be the case if we followed the classic phased-elimination design, which does not utilize any information from previous phases. However, in every phase, \spe~leverages the side information obtained from previous phase. We thus make the observation that if action 0 is indeed sub-optimal and let $\Delta = \max_{a\in\cA}\bar{r}(a)-\bar{r}(0)$, then after $O(\log((\Delta+(1-\alpha)\bar{r}(0))^{-1}))$ phases (as we have seen in \cref{lemma:spe3} of \cref{thm:safety_regret}), we would have that any plausible action $a$ would satisfy $L_h(a)\geq \alpha\bar{r}(0).$ Consequently, the safe optimal design in Step 1 would no longer allocate probability mass to action 0. In other words, it is eliminated implicitly. 
		\item Following the previous point, if we do not use the side information in Step 1, the resulting algorithm would always select the default action 0 for $\alpha$ fraction of the time. Under this, sub-optimal actions, except for the default action 0, would be gradually removed. Nevertheless, the default action 0 might incur unnecesarily large regret.
	\end{enumerate}

\section{Conclusions}
In this work, we design safe optimal logging policy to simultaneously collect high-quality data for off-policy learning and achieve competitive expected reward when compared to a production policy. We first show that the policy induced by mixing the production policy and uniform exploration is safe but sub-optimal in general. Then, we develop optimal solutions for a variety of cases and discuss their implications for off-policy evaluation and optimization. Finally, we apply our techniques to develop optimal safe online learning algorithm with small number of policy updates.

\bibliographystyle{ormsv080}
\bibliography{ref}

\clearpage
\begin{appendices}{\Large \noindent\textbf{Supplementary and Proofs}}

\section{Supplements to \cref{sec:mixing}}\label{sec:supplement}
The expected reward of $\pi_{\beta}$ is
\begin{align}
    V(\pi_{\beta})=\sum_{a\in\cA}\pi_{\beta}(a)\bar{r}(a)=\sum_{a\in\cA}\left[\left(\beta+\frac{1-\beta}{K\pi_0(a)}\right)\pi_0(a)\right]\bar{r}(a).
\end{align}  
According to \eqref{eq:safety_equivalent}, we can deduce that one needs to ensure
\begin{align}
    \beta+\frac{1-\beta}{K\pi_0(a)}\geq\alpha\qquad\forall a\in\cA.
\end{align}
Here, the binding constraint is 
\begin{align}
    \beta+\frac{1-\beta}{K\max_{a\in\cA}\pi_0(a)}\geq\alpha\qquad\Rightarrow\qquad \beta\geq\frac{\alpha-(K\max_{a\in\cA}\pi_0(a))^{-1}}{1-(K\max_{a\in\cA}\pi_0(a))^{-1}}\,.
\end{align}
Depending on the value of $\alpha,$ we have
\begin{align}
    \beta\geq\beta_{*}:=\max\left\{\frac{\alpha-(K\max_{a\in\cA}\pi_0(a))^{-1}}{1-(K\max_{a\in\cA}\pi_0(a))^{-1}},0\right\}\,.
\end{align} 
Consider the objective,
\begin{align}
 g(\pi_{\beta})=\max_{a\in\cA}\frac{1}{\pi_{\beta}(a)}=&\frac{1}{\min_{a\in\cA}\beta\pi_0(a)+(1-\beta)/K}
    \geq\frac{1}{\min_{a\in\cA}\beta_{*}\pi_0(a)+(1-\beta_{*})/K},\label{eq:mix_unif00}
\end{align}
where the last step follows from the fact that $\min_{a\in\cA}\pi_0(a)\leq1/K$ and we should thus take $\beta=\beta_{*}$ (which makes the last step in \eqref{eq:mix_unif00} to equality) to minimize $g(\pi_{\beta}).$

In general, we show in the following proposition that the logging policy $\pi_{\beta_{*}}$ is optimal when $\pi_0$ only takes two different values or $\alpha\leq(K\max_{a\in\cA}\pi_0(a))^{-1}$, but it is not optimal otherwise.
\begin{proposition}\label{prop:mix_unif} Let $M$ be the different values of $\pi_0$ (\ie, the cardinality of the set $\{\pi_0(a)\}_{a\in\cA}$) the following two statements hold for $\pi_{\beta_{*}}:$ 
\begin{enumerate}
\item If $M\leq2$ or $\alpha\leq (K\max_{a\in\cA}\pi_0(a))^{-1},$ then for any policy $\pi$ such that the safety constraint $V(\pi)\geq\alpha V(\pi_0)$ holds for all $\bar{r}\in[0,1]^K,$ we have $g(\pi)\geq g(\pi_{\beta_{*}});$
\item If $M\geq 3$ and $\alpha>(K\max_{a\in\cA}\pi_0(a))^{-1},$ there exists a reward parameter $\theta$ and a policy $\pi$ such that the safety constraint $V(\pi)\geq\alpha V(\pi_0)$ holds for all $\bar{r}\in[0,1]^K$, but $g(\pi)<g(\pi_{\beta_{*}}).$
\end{enumerate}
\end{proposition}
\begin{proof}
For the first part, it is evident that $\pi_{\beta_{*}}$ is optimal when $\pi_0(\cdot)$ only takes two different values as the probability mass of action(s) $\argmax_{a\in\cA}\pi_{\beta_{*}}(a)$ could not be further reduced to preserve the safety constraint. Also note that when $\alpha\leq(K\max_{a\in\cA}\pi_0(a))^{-1}$, we can set $\beta=0~(=\beta_{*})$ to ensure $g(\pi_{\beta_{*}})=K,$ which is the optimal value of the objective value even without the safety constraint.

For the second part, w.l.o.g., we let $\delta$ be any number in $(0,1]$ and $$\pi_0(1)=\frac{1-\delta}{K},\quad \pi_0(2)=\ldots=\pi_0(K-1)=\frac{1}{K},\quad\pi_0(K)=\frac{1+\delta}{K}.$$ 
Then for any $\alpha>(K\pi_0(K))^{-1}=(K(1+\delta)/K)^{-1}=(1+\delta)^{-1},$ we have
\begin{align}
    \beta_{*}=\frac{1-\alpha}{1-\frac{1}{K\pi_0(k)}}=\frac{1-\alpha}{1-(1+\delta)^{-1}}=\frac{(1-\alpha)(1+\delta)}{\delta}
\end{align}
and can easily verify that $\pi_{\beta_{*}}(k)=(1-\beta_{*})\pi_0(k)+\beta_{*}/K=\alpha\pi_0(k).$ However, for any $k\in[2,K-1],$ we have 
\begin{align}
    \pi_{\beta_{*}}(k)=(1-\beta_{*})\pi_0(k)+\frac{\beta_{*}}{K}=\frac{1}{K}=\pi_0(k)>\alpha\pi_0(k)
\end{align}
and 
\begin{align}
    \pi_{\beta_{*}}(1)=(1-\beta_{*})\pi_0(1)+\frac{\beta_{*}}{K}=\frac{1-\delta+(1-\alpha)(1+\delta)}{K}=\frac{2-\alpha(1+\delta)}{K}<\frac{1}{K},
\end{align}
where the last step follows from the precondition $\alpha>(1+\delta)^{-1}.$ Therefore, $$g(\pi_{\beta_{*}})=\frac{K}{2-\alpha(1+\delta)}.$$
Now, consider $\zeta=\min\{1-\alpha,(\alpha(1+\delta)-1)/(K-1)\}$ and a policy $\pi,$ such that 
\begin{align}
    \pi(1)=\frac{(K-2)\zeta}{K}+\pi_{\beta_{*}}(1),\quad\pi(2)=\ldots\pi(a_{K-1})=\frac{1-\zeta}{K},\quad\pi(k)=\pi_{\beta_{*}}(k).
\end{align}
It is evident that 
$$\pi(2)-\pi(1)=\frac{1-(K-1)\zeta}{K}-\pi_{\beta_{*}}(1)\geq\frac{1-\alpha(1+\delta)+1}{K}-\pi_{\beta_{*}}(1)=0$$ and hence,
$g(\pi)=1/\pi(1)<g(\pi_{\beta_{*}}).$
\end{proof}

\section{Proof of  \cref{thm:greedy}}\label{sec:thm:greedy}
For $\alpha=0,$ the statement is trivial and for $\alpha>0,$ we prove the claim by contradiction. For the policy $\pie(\cdot)$ returned by the water-filling method, we can first sort the actions w.l.o.g. as follows:
\begin{align}
    \pie(1)\leq\pie(2)\leq\ldots\leq\pie(K).
\end{align}
Due to the nature of water filling, we can find $k\in[K]$ such that $\pie(s)>\alpha\pi_0(s)$ holds for all $s\leq k$ while $\pie(s)=\alpha\pi_0(s)$ for all $s>k.$ Critically,
\begin{align}
    \pie(1)=\pie(2)=\ldots=\pie(k);
\end{align}

Suppose that there exists a policy $\pi$ such that $\pi(a)\geq \alpha \pi_0(a)$ holds for all $a\in\cA$ while $\min_{a\in\cA}\pi(a)>\min_{a\in\cA}\pie(a).$ Now consider the smallest $k'\in[K]$ such that $\pi(k')<\pie(k')$ (note that this $k'$ is guaranteed to exist because $\sum_{a\in\cA}\pi(a)=\sum_{a\in\cA}\pie(a)=1$). We distinguish two cases:
\begin{itemize}
    \item \textbf{Case 1.} $k'\leq k:$ In this case, we have that 
    \begin{align}
        \min_{a\in\cA}{\pi}(a)\leq\pi(k')<\pie(k')=\min_{a\in\cA}\pie(a)
    \end{align}
    by definition of $k'$ and $k,$ which is a contradiction to $\min_{a\in\cA}\pi(a)>\min_{a\in\cA}\pie(a).$
    \item \textbf{Case 2.} $k'>k:$ In this case, we have that
    \begin{align}
        \pi(k')<\pie(k')=\alpha\pi_0(k')\,,
    \end{align} 
    which is a contradiction to $\pi(a)\geq \alpha \pi_0(a)$ for all $a\in\cA.$
\end{itemize}
Consequently, we conclude the proof.

\section{Proof of Theorem \ref{thm:duality}}\label{sec:thm:duality}
First, suppose $\pie$ is a feasible solution for $\text{P}_1(L,U,\pi_0).$ Then, it holds $\min_{\bar{r}\in[L,U]}(\pie-\alpha\pi_0)^{\top}\bar{r}\geq0$and hence, by strong duality between $\text{P}_2(L,U,\pi_0,\pie)$ and $\text{D}_2(L,U,\pi_0,\pie),$ there exist $z_1,z_2\geq0,$ such that $L^{\top}z_1-U^{\top}z_2\geq0$ and $z_1-z_2=\pie-\alpha\pi_0.$ Consequently, $(\pie,z_1,z_2)$ is a feasible solution for  $\text{P}_1(L,U,\pi_0),$ which indicates the optimal value of $\text{P}_3(L,U,\pi_0)$ is at least that of $\text{P}_1(L,U,\pi_0)$.

Conversely, we can also show that the optimal value of $\text{P}_1(L,U,\pi_0)$ is at least that of $\text{P}_3(L,U,\pi_0)$, which would complete the proof of the statement.

\section{Supplements to \cref{remark:tabular_side}}\label{sec:remark:tabular_side}
By linearity, we know that 
$$\min_{\bar{r}\in[L,U]}(\pie-\alpha\pi_0)^{\top}\bar{r}=\sum_{a\in\cA}\min\left\{(\pie(a)-\alpha\pi_0(a))L(a),(\pie(a)-\alpha\pi_0(a))U(a)\right\}\,.$$
It is thus evident that if $\pie$ is feasible for $\text{P}_1(L,U,\pi_0),$ we can set
$$z(a)=\min\left\{(\pie(a)-\alpha\pi_0(a))L(a),(\pie(a)-\alpha\pi_0(a))U(a)\right\}\quad \forall a\in\cA\,,$$
and $(\pie,z)$ would be a feasible solution for $\text{P}'_1(L,U,\pi_0).$ 

Conversely, if $(\pie,z)$ is a feasible solution for $\text{P}'_1(L,U,\pi_0),$ $\pie$ would also be feasible for $\text{P}_1(L,U,\pi_0).$ Otherwise, if there exists $\bar{r}_0\in[L,U]$ such that $(\pie-\alpha\pi_0)^{\top}\bar{r}_0<0,$ then
$$z^{\top}\mathbf{1}_K\leq\sum_{a\in\cA}\min\left\{(\pie(a)-\alpha\pi_0(a))L(a),(\pie(a)-\alpha\pi_0(a))U(a)\right\}=\min_{\bar{r}\in[L,U]}(\pie-\alpha\pi_0)^{\top}\bar{r}<0\,,$$ which leads to a contradiction.

\section{Proof of \cref{lemma:ips_error_bound}}\label{sec:lemma:ips_error_bound}
Although it is straightforward to show that the confidence interval holds for a single fixed policy $\pi,$ directly applying the union bound over the entire policy space $\Pi$ would not lead to the desired statement because $\Pi$ contains infinitely many policies.

To overcome this challenge, we use singleton policies $\{\pi^{(k)}(\cdot\mid x)\}$, where $\pi^{(k)}(\cdot\mid x)$ assign probability $1$ to action $k$ and $0$ to the rest for each context $x,$ \ie,
\begin{align*}
    \pi^{(k)}(a\mid x)=\begin{cases}
  1  &  \text{when } a=k \\
  0 &  \text{otherwise.}
\end{cases}
\end{align*} 
Specifically, we extend the IPS estimator to the conditional expected reward, \ie,
\begin{align*}
    \hat{V}(\pi(\cdot\mid x))=\frac{1}{n_x}\sum_{t=1}^n\frac{\pi(a_t\mid x_t)\I{x_t=x}}{\pie(a_t\mid x_t)}r_t,
\end{align*}
where $n_x=\sum_{t=1}^n\I{x_t=x}$ is the number of times that the context $x$ is recorded. It is evident that \begin{align*}
\E\left[\frac{\pi(a_t\mid x_t)\I{x_t=x}}{\pie(a_t\mid x_t)}r_t\right]=V(\pi(\cdot\mid x))
\end{align*}
and each term in this IPS estimator is $g(\pie)^2/4$-sub-Gaussian. Therefore, conditioned on $x,n_x,$ and for a policy $\pi(\cdot\mid x)$, by Hoeffding's inequality, we have with probability at least $1-\delta,$
\begin{align}
    \left|\hat{V}(\pi(\cdot\mid x))-V(\pi(\cdot\mid x))\right|\leq g(\pie)\sqrt{\frac{\log(2/\delta)}{2\max\{1,n_x\}}}.
\end{align}
By a union bound over all possible $x\in\cX,n_x~(\leq n),$ and $k\in[K],$ we have with probability at least $1-\delta,$ for every possible $x\in\cX$
\begin{align}\label{eq:simplex_ope_new1}
    \max_{k\in[K]}\left|\hat{V}(\pi^{(k)}(\cdot\mid x))-V(\pi^{(k)}(\cdot\mid x))\right|\leq g(\pie)\sqrt{\frac{\log(2K|\cX|n/\delta)}{2\max\{1,n_x\}}}.
\end{align} 

Note that for every policy $\pi(\cdot\mid x),$ there exists coefficients $u_1,\ldots,u_K\in[0,1]$ such that  $\pi(\cdot\mid x)=\sum_{k\in[K]}u_k\pi^{(k)}(\cdot\mid x).$ Hence, 
for any $\pi(\cdot\mid x),$ we have 
\begin{align*}
   \hat{V}(\pi(\cdot\mid x))=&\frac{1}{n_x}\sum_{t=1}^n\frac{\pi(a_t\mid x_t)\I{x_t=x}}{\pie(a_t\mid x_t)}r_t\\
   =&\frac{1}{n_x}\sum_{t=1}^n\frac{\sum_{k\in[K]}u_{k}\pi^{(k)}(a_t\mid x_t)\I{x_t=x}}{\pie(a_t\mid x_t)}r_t=\hat{V}\left(\sum_{k\in[K]}u_k\pi^{(k)}(\cdot\mid x)\right)
\end{align*}
and 
\begin{align*}
    V(\pi(\cdot\mid x))=\sum_{a\in\cA}\pi(a\mid x)\bar{r}(x,a)=\sum_{a\in\cA}\left(\sum_{k\in[K]}\pi^{(k)}(a\mid x)\right)\bar{r}(x,a)=V\left(\sum_{k\in[K]}u_k\pi^{(k)}(\cdot\mid x)\right)
\end{align*}
Therefore, for every context $x$ and every policy $\pi(\cdot\mid x),$
\begin{align*}
    \left|\hat{V}(\pi(\cdot\mid x))-  V(\pi(\cdot\mid x))\right|=&\left|\hat{V}\left(\sum_{k\in[K]}u_k\pi^{(k)}(\cdot\mid x)\right)-V\left(\sum_{k\in[K]}u_k\pi^{(k)}(\cdot\mid x)\right)\right|\\
    =&\left|\sum_{k\in[K]}u_k\hat{V}\left(\pi^{(k)}(\cdot\mid x)\right)-\sum_{k\in[K]}u_kV\left(\pi^{(k)}(\cdot\mid x)\right)\right|\\
    \leq&\sum_{k\in[K]}u_k\left|\hat{V}\left(\pi^{(k)}(\cdot\mid x)\right)-V\left(\pi^{(k)}(\cdot\mid x)\right)\right|\\
    \leq& g(\pie)\sqrt{\frac{\log(2K|\cX|n/\delta)}{2\max\{1,n_x\}}}
\end{align*} 
with probability at least $1-\delta.$ Here, the second step utilizes the linearity of $\hat{V}(\cdot)$ and $V(\cdot),$ the third step follows from triangle inequality, and the final step follows from \eqref{eq:simplex_ope_new1}. This further implies for every $x\in\cX$ and every $\pi(\cdot\mid x),$
\begin{align}
    \label{eq:simplex_ope_new2}\left|n_x\hat{V}(\pi(\cdot\mid x))- n_x V(\pi(\cdot\mid x))\right|\leq n_x g(\pie)\sqrt{\frac{\log(2K|\cX|n/\delta)}{2\max\{1,n_x\}}}
\end{align}
holds with probability at least $1-\delta.$ Consequently, we have for every $\pi,$
\begin{align}
    \nonumber&\left|\hat{V}(\pi)-V(\pi)\right| \\
    \nonumber=& \left|\sum_{x\in\cX}\frac{n_x}{n}\hat{V}(\pi(\cdot\mid x))-\sum_{x\in\cX}\cC(x)V(\pi(\cdot\mid x))\right|\\
    \nonumber=& \left|\sum_{x\in\cX}\frac{n_x}{n}\hat{V}(\pi(\cdot\mid x))-\sum_{x\in\cX}\frac{n_x}{n}V(\pi(\cdot\mid x))+\sum_{x\in\cX}\frac{n_x}{n}V(\pi(\cdot\mid x))-\sum_{x\in\cX}\cC(x)V(\pi(\cdot\mid x))\right|\\
    \nonumber\leq&\left|\sum_{x\in\cX}\frac{n_x}{n}\hat{V}(\pi(\cdot\mid x))-\sum_{x\in\cX}\frac{n_x}{n}V(\pi(\cdot\mid x))\right|+\left|\sum_{x\in\cX}\frac{n_x}{n}V(\pi(\cdot\mid x))-\sum_{x\in\cX}\cC(x)V(\pi(\cdot\mid x))\right|\\
    \nonumber\leq&\sum_{x\in\cX}\frac{1}{n}\left|n_x\hat{V}(\pi(\cdot\mid x))-n_x V(\pi(\cdot\mid x))\right|+\left|\sum_{x\in\cX}\frac{n_x}{n}V(\pi(\cdot\mid x))-\sum_{x\in\cX}\cC(x)V(\pi(\cdot\mid x))\right|\\
    \leq&\frac{g(\pie)}{n}\sqrt{\frac{\log(2K|\cX|n/\delta)}{2}}\sum_{x\in\cX}\sqrt{\frac{n^2_x}{\max\{1,n_x\}}}+\left|\sum_{x\in\cX}\frac{n_x}{n}V(\pi(\cdot\mid x))-\sum_{x\in\cX}\cC(x)V(\pi(\cdot\mid x))\right|,
    \label{eq:simplex_ope_new3}
\end{align}
with probability at least $1-\delta.$ Here, the third and fourth steps follow from triangle inequality and the last step utilizes \eqref{eq:simplex_ope_new2}.

For the first term of \eqref{eq:simplex_ope_new3}, we have
\begin{align}\label{eq:simplex_ope_new4}
    \sum_{x\in\cX}\sqrt{\frac{n^2_x}{\max\{1,n_x\}}}\leq\sqrt{\left(\sum_{x\in\cX}n_x\I{n_x>0}\right)\left(\sum_{x\in\cX}\I{n_x>0}\right)}\leq\sqrt{n|\cX|}
\end{align}
where the first step follows from the Cauchy-Schwarz inequality.

For the second term of \eqref{eq:simplex_ope_new3}, we have with probability $1-\delta,$
\begin{align}\label{eq:simplex_ope_new5}
    \nonumber\left|\sum_{x\in\cX}\frac{n_x}{n}V(\pi(\cdot\mid x))-\sum_{x\in\cX}\cC(x)V(\pi(\cdot\mid x))\right|\leq&\left|\sum_{x\in\cX}\frac{n_x}{n}-\sum_{x\in\cX}\cC(x)\right|\max_{x\in\cX}V(\cdot\mid x)\\
    \leq&\sqrt{\frac{14|\cX|\log(2n/\delta)}{n}},
\end{align}
where the first step follows from the H\"{o}lder's inequality and the second step follows from the $\ell_1$ deviation inequality \citep{WeissmanOSVW03}.

Combining \eqref{eq:simplex_ope_new3},\eqref{eq:simplex_ope_new4}, and \eqref{eq:simplex_ope_new5}, we have with probability at least $1-\delta,$ 
\begin{align*}
     \max_{\pi}\left|\hat{V}(\pi)-V(\pi)\right|\leq 7g(\pie)\sqrt{\frac{|\cX|\log(4K|\cX|n/\delta)}{2n}}.
\end{align*}

The second claim follows from $\hat{V}(\pi)$ being close to $V(\pi)$ for any policy $\pi$ with a high probability, \ie,
\begin{align*}
  V(\pi_*) - V(\hat{\pi})
  = &V(\pi_*) - \hat{V}(\pi_*) + \hat{V}(\pi_*) - V(\hat{\pi})
  \\
  \leq &V(\pi_*) - \hat{V}(\pi_*) + \hat{V}(\hat{\pi}) - V(\hat{\pi})
  \leq 14g(\pie)\sqrt{\frac{|\cX|\log(4K|\cX|n/\delta)}{2n}}\,.
\end{align*}

\section{Proof of \cref{lemma:pi_error_bound}}\label{sec:lemma:pi_error_bound}
Differently from \cref{lemma:ips_error_bound}, we can no longer consider singleton policies (as they might not even exist in $\cA$). To deal with the potentially large action set, we consider a discretization over the space of $\cS^{|\cX|}$ where $\cS=\{a\in\realset^d:\|a\|_2\leq1\}.$ Specifically, we let $\cQ$ be the $\min\{1,\sqrt{\lambda_*}\}/n$-cover of $\cS$ (\ie, for any $a\in\cS,$ there exists $a'\in\cQ$ such that $\|a-a'\|_2\leq 1/n$). Then we know that $\left|\cQ\right|\leq(3\min\{1,\sqrt{\lambda_*}\}/n)^d$ (which implies $\left|\cQ^{|\cX|}\right|\leq(3\min\{1,\sqrt{\lambda_*}\}/n)^{d|\cX|}$). To proceed, we define $A_{\cQ}$ to be the matrix that contains each element of $\cQ$ as its column. With slight abuse of notation, we define the set of all possible deterministic policies that map a context $x$ to an action in $\cQ$ as $\cH=\{h:\cX\to\cQ\}$. Then by Hoeffding's inequality (see \eg, equations (5.6) and (5.7) of \cite{LS18}) and union bound, with probability at least $1-\delta,$
\begin{align}
 \label{eq:linear_ope0}\max_{h\in\cH}\left|\hat{V}(h)-V(h)\right|\leq g(\pie)\sqrt{\frac{d|\cX|\log(n/(\delta\min\{1,\sqrt{\lambda_*}\}))}{2n}}\,.
\end{align} 
 Now for any policy $\pi\in\Pi,$ we find the policy $h_{\pi}\in\cH$ such that for every $x\in\cX,$ 
 \begin{align}\label{eq:linear_ope4}
 \|A\pi(\cdot\mid x)-A_{\cQ}h_{\pi}(\cdot|x)\|_2\leq \frac{\min\{1,\sqrt{\lambda_*}\}}{n}.
 \end{align} 
 Then, we have
\begin{align}
    \nonumber\left|\hat{V}(\pi)-V(\pi)\right|=&\left|\hat{V}(\pi)-\hat{V}(h_{\pi})+\hat{V}(h_{\pi})-V(h_{\pi})+V(h_{\pi})-V(\pi)\right|\\
    \label{eq:linear_ope1}\leq &\left|\hat{V}(\pi)-\hat{V}(h_{\pi})\right|+\left|\hat{V}(h_{\pi})-V(h_{\pi})\right|+\left|V(h_{\pi})-V(\pi)\right|\,,
\end{align}
where the second step follows from the triangle inequality. Now for the first term of \eqref{eq:linear_ope1}
\begin{align}
    \nonumber\left|\hat{V}(\pi)-\hat{V}(h_{\pi})\right|=&\left|\frac{1}{n}\sum_{t=1}^n r_t\cdot\left(A\pi(\cdot\mid x_t)-A_{\cQ}h_{\pi}(\cdot\mid x_t)\right)^{\top}G(\pie(\cdot \mid x_t))^{-1}a_t\right|\\
    \nonumber\leq&\frac{1}{n}\sum_{t=1}^n \left|\left(A\pi(\cdot\mid x_t)-A_{\cQ}h_{\pi}(\cdot\mid x_t)\right)^{\top}G(\pie(\cdot \mid x_t))^{-1}a_t\right|\\
    \nonumber\leq&\frac{1}{n}\sum_{t=1}^n \sqrt{\left(A\pi(\cdot\mid x_t)-A_{\cQ}h_{\pi}(\cdot\mid x_t)\right)^{\top}G(\pie(\cdot \mid x_t))^{-1}\left(A\pi(\cdot\mid x_t)-A_{\cQ}h_{\pi}(\cdot\mid x_t)\right)}\\
    \nonumber&\quad\times\sqrt{a^{\top}_{t}G(\pie(\cdot \mid x_t))^{-1}a_t}\\
    \nonumber\leq&\frac{1}{n}\sum_{t=1}^n\sqrt{\lambda_{\max}\left(G\left(\pie(\cdot\mid x_t)^{-1}\right)\right)\|A\pi(\cdot\mid x_t)-A_{\cQ}h_{\pi}(\cdot\mid x_t)\|^2_2}\sqrt{g(\pie)}\\
    \leq &\frac{\sqrt{g(\pie)}}{n}\,,
\end{align} 
where the first inequality follows from the triangle inequality and $r_t \in [0, 1]$, the second one follows from Cauchy-Schwarz inequality, and the third one follows from the fact that $a^{\top}Qa\leq\lambda_{\max}(Q)\|a\|^2_2$ for any symmetric positive semi-definite matrix $Q$. In the last equality, we use the premise that $\lambda_{\min}(G(\pi_e(\cdot\mid x)))\geq\lambda_*$ (or equivalently, $\lambda_{\max}(G^{-1}(\pi_e(\cdot\mid x)))\leq\lambda_*$) and \eqref{eq:linear_ope4}. 
The second term of \eqref{eq:linear_ope1} can be easily bounded from above using \eqref{eq:linear_ope0} as
\begin{align}
    \left|\hat{V}(h_{\pi})-V(h_{\pi})\right|\leq g(\pie)\sqrt{\frac{d|\cX|\log(n/(\delta\min\{1,\sqrt{\lambda_*}\}\}))}{2n}}\,.
\end{align}
For the third term, by triangle inequality and Cauchy-Schwarz inequality,
\begin{align}
    \nonumber\left|V(h_{\pi})-V(\pi)\right|=&\left|\sum_{x\in\cX}\cC(x)(A\pi(\cdot\mid x)-A_{\cQ}h_{\pi}(\cdot\mid x))^{\top}\theta_{*,x}\right|\\
    \nonumber\leq&\sum_{x\in\cX}\cC(x)\left|(A\pi(\cdot\mid x)-A_{\cQ}h_{\pi}(\cdot\mid x))^{\top}\theta_{*,x}\right|\\
    \nonumber\leq&\sum_{x\in\cX}\cC(x)\left\|(A\pi(\cdot\mid x)-A_{\cQ}h_{\pi}(\cdot\mid x))\right\|_2\left\|\theta_{*,x}\right\|_2\\
    \leq&\frac{1}{n}\,,
\end{align}
where the last step follows from \eqref{eq:linear_ope4} and that $\|\theta_{*,x}\|\leq 1.$

Combining the above, we have
\begin{align*}
    \max_{\pi\in\Pi}\left|\hat{V}(\pi)-V(\pi)\right|\leq3 g(\pie)\sqrt{\frac{d|\cX|\log(n/(\delta\min\{1,\sqrt{\lambda_*}\}\}))}{2n}}\,.
\end{align*}

The proof of the second part is very similar to that of Lemma \ref{lemma:ips_error_bound} and is thus omitted.

\section{Supplements for Proof of \cref{thm:safety_regret}}\label{sec:thm:safety_regret}

\subsection{Proof of \cref{lemma:spe0}}\label{sec:lemma:spe0}

By Hoeffding's inequality (see \eg, (5.6) and (5.7) of \cite{LS18}), we know that for any given starting round $t_h$ of phase $h,$ action $a,$ and $N_h(a),$ we have 
\begin{align*}
	\Pr\left(|\bar{r}(a)-\hat{r}_h(a)|\leq\sqrt{\frac{\log(KT^4/\delta)}{2N_h(a)}}\right)\geq 1-\frac{\delta}{KT^{4}}.
\end{align*}
Now by a union bound over all possible starting round $t_h\in[T]$ of phase $h,$ actions $a\in\cA_h\subseteq([K]\cup\{0\}),$ and the corresponding $N_h(a),$ we have that 
\begin{align*}
	\Pr(\bar{r}(a)\in[L_h(a),U_h(a)]~\forall h~\forall a\in\cA_h)\geq 1-\frac{\delta}{T}>1-\delta.
\end{align*}
For the second part, 
\begin{align}\label{eq:spe2}
	\nonumber\sqrt{\frac{\log(KT^4)}{2N_h(a)}}=&\sqrt{\frac{\log(KT^4/\delta)}{\pi_h(a)g_h(\pi_h)\epsilon_h^{-2}\log(KT^4/\delta)}}\\
	=&\epsilon_h\sqrt{\frac{1}{\pi_h(a)\max_{a'\in\cA_h\setminus\{0\}}\pi^{-1}_h(a')}}\leq\epsilon_h\sqrt{\frac{1}{\pi_h(a)\pi_h(a)^{-1}}}=\epsilon_h,
\end{align}
which indicates  $[L_h(a),U_h(a)]\subseteq[\hat{r}_h(a)-\epsilon_h,\hat{r}_h(a)+\epsilon_h]$ for all $h$ and $a\in\cA_h.$

\subsection{Proof of \cref{lemma:spe2}}\label{sec:lemma:spe2}
For any phase $h,$ let $a^h=\argmax_a\hat{r}_h(a),$ we always have
\begin{align*}
	\hat{r}_h(a^{opt})\geq \bar{r}(a^{opt})-\epsilon_h \geq \bar{r}(a^h)-\epsilon_h\geq\hat{r}_h(a^h)-2\epsilon_h
\end{align*}
Here, the first and last step follows from $\cE,$ the second step follows from the fact that $\bar{r}(a^{opt})\geq\bar{r}(a^h).$ Therefore, $a^{opt}$ would not be removed.

Next, for each action $a,$ consider the first phase $h(a)$ such that $\Delta(a)\geq5\epsilon_{h(a)}$ (one can also compute that $h(a)=\lceil\log_2(5/\Delta(a))\rceil$), if action $a$ is not removed before, we have
\begin{align*}
	\max_{a‘\in\cA_{h(a)}}\hat{r}_{h(a)}(a’)-2\epsilon_{h(a)}\geq&\hat{r}_{h(a)}(a^{opt})-2\epsilon_{h(a)}
	\geq\bar{r}(a^{opt})-3\epsilon_{h(a)}\\
	=&\bar{r}(a)+{\Delta(a)}-3\epsilon_{h(a)}
	\geq\hat{r}_{h(a)}(a)+\Delta(a)-4\epsilon_{h(a)}>\hat{r}_{h(a)}(a).
\end{align*}
Here, the first inequality follows from above that the optimal action is never removed, the second and third inequalities follow from $\cE,$ and the last inequality follows by definition of $h(a).$

\subsection{Proof of \cref{lemma:spe6}}\label{sec:lemma:spe6}
we first make the following two observations:
\begin{enumerate}
    \item For an action $a\in[K],$ if $\Delta(a)\leq \sqrt{K/T},$ then even if action $a$ is selected $T$ times, the regret is at most $\sqrt{KT}.$ We thus only focus on actions whose $\Delta(a)$ is at least $\sqrt{K/T};$
    \item For an action $a\in[K],$ if $\bar{r}(0)\leq L_h(a)~(\leq \bar{r}(a))$ for some $h,$ it must be that $\bar{r}(a)=\bar{r}(0)$ due to optimality of $\bar{r}(0).$
\end{enumerate}
Following this,
\begin{align}
	\nonumber&\regret(\spe)\\
	\nonumber\leq&\underbrace{\sqrt{KT}}_{\text{For actions with }\Delta(a)\in\left(0,\sqrt{K/T}\right]}+\sum_{h=1}^{\min_{a:\Delta(a)\geq\sqrt{K/T}}\log_2(5/\Delta(a))}\sum_{a\in\cA_h}\Delta(a)N_h(a)\\
	\nonumber\leq&{\sqrt{KT}}+\sum_{h=1}^{\log_2(5\sqrt{T}/\sqrt{K})}\sum_{a\in\cA_h}\Delta(a)0.5\pi_h(a)g_h(\pi_h)\epsilon_h^{-2}\log(KT^4/\delta)\\		\nonumber<&{\sqrt{KT}}+\sum_{h=1}^{\log_2(5\sqrt{T}/\sqrt{K})}\sum_{a\in\cA_h,\Delta(a)>0}5\epsilon_h0.5\pi_h(a)g_h(\pi_h)\epsilon_h^{-2}\log(KT^4/\delta)\\
	\label{eq:spe1}\leq&{\sqrt{KT}}+3\max_h\left[\left(1-\pi_h(0)-\sum_{a:L_h(a)\geq\bar{r}(0)}\pi_h(a)\right)g_h(\pi_h)\right]\log(KT^4/\delta)\sum_{h=1}^{\log_2(5\sqrt{T}/\sqrt{K})}\epsilon_h^{-1}.
\end{align}
Here, the first inequality upper bounds the regret by distinguishing actions with $\Delta(a)\leq\sqrt{K/T}$ and actions with $\Delta(a)>\sqrt{K/T}$ (as for actions with $\Delta(a)\leq\sqrt{K/T},$ the regret incurred by them is at most $T\sqrt{K/T}=\sqrt{KT}$). The second step makes use of the definition of $N_h(a),$ which is the number of times that action $a$ is chosen in phase $h.$ The third step applies \cref{lemma:spe2}. 

The conclusion follows from
\begin{align*}
\sum_{h=1}^{\log_2(5\sqrt{T}/\sqrt{K})}\epsilon_h^{-1}=\sum_{h=1}^{\log_2(5\sqrt{T}/\sqrt{K})}2^{h}<2^{\log_2(5\sqrt{T}/\sqrt{K})+1}<5\sqrt{\frac{T}{K}}+1,
\end{align*}
where the third step uses the fact that $\sum_{s=1}^{u}2^s\leq 2^{u+1}$.

\subsection{Proof of \cref{lemma:spe5}}\label{sec:lemma:spe5}
On $\cE,$ we have for all $h$ and $a,$ it holds that
$L_h(a)<\bar{r}(0)$ if $\bar{r}(0)>\bar{r}(a)$ (which was prescribed by the presumption of this case). Under this, \spe~would allocate equal probability mass to all actions in $\cA_h\setminus\{0\}.$ 

Assuming this is not the case, let $a^1=\argmin_a\pi_h(a)$ and suppose there exists $a^2\in\cA_h\setminus\{0\}$ such that $\pi_h(a^1)<\pi_h(a^2),$ then
\begin{itemize}
    \item If $L_h(a^1)\geq L_h(a^2),$ we can move $(\pi_h(a^2)-\pi_h(a^1))/2$ from $\pi_h(a^2)$ to $\pi_h(a^1)$. This would only decrease $g_h(\pi_h);$
    \item If $L_h(a^1)< L_h(a^2),$ we could  first move probability mass $\pi_h(a^2)-\pi_h(a^1)$ from $\pi_h(a^2)$ to action 0, which would unbind the safety constraint \eqref{eq:tabular_spe_safety}. Then one can follow a step similar to water-filling to re-distribute the (new) probability mass to $a^1$ and $a^2$ such that the safety constraint is met, but $g_h(\pi_h)$ is further decreased.
\end{itemize}
Both of the cases would contradict the optimality of $\pi_h$ to \eqref{eq:tabular_spe_safety}. Therefore,
\begin{align*}
	&\left(1-\pi_h(0)-\sum_{a:L_h(a)\geq \bar{r}(0)}\pi_h(a)\right)g_h(\pi_h)\\
	\leq&\left(1-\pi_h(0)-\sum_{a:L_h(a)\geq \bar{r}(0)}\pi_h(a)\right)\frac{|\cA_h|-1}{1-\pi_h(0)-\sum_{a:L_h(a)\geq \bar{r}(0)}\pi_h(a)}\\
	<&|\cA_h|-1<K.
\end{align*}

\subsection{Proof of \cref{lemma:spe3}}\label{sec:lemma:spe3}
We consider the end of any phase $h$ (\ie, after action elimination step), for any action $a\in\cA_{h+1}\setminus\{0\},$ we have
\begin{align*}
	L_{(h+1)}(a) =& \hat{r}_{h}(a)-\sqrt{\frac{\log(KT^4/\delta)}{N_{h}(a)}}
	\geq\hat{r}_{h}(a) - \epsilon_{h}\\
	\geq&\max_{a'\in\cA_{h}}\hat{r}_{h}(a') - 3\epsilon_{h}\geq\hat{r}_{h}(a^{opt}) - 3\epsilon_{h}\geq\bar{r}(a^{opt}) - 4\epsilon_{h}.
\end{align*}
Here, the first step follows by definition, the second step follows from \eqref{eq:spe2}, the third step follows from the fact that $a$ is not removed after phase $h,$ the forth step follows by \cref{lemma:spe2} that $a^{opt}$ won't be removed under $\cE$, and the last step follows from $\cE.$ 

Note that $\bar{r}(0)=\bar{r}(a^{opt})- \Delta(0),$ we have $L_{(h+1)}(a)\geq\bar{r}(0)+\Delta(0)-4\epsilon_{h}.$ This indicates that for any phase $h$ such that $\Delta(0)+(1-\alpha)\bar{r}(0)> 4\epsilon_{h}$ (or $h> \log_2(4/(\Delta(0)+(1-\alpha)\bar{r}(0)))$), one has
\begin{align*}
	L_{(h+1)}(a) - \alpha\bar{r}(0)\geq\bar{r}(a^{opt}) - 4\epsilon_{h}-\alpha\bar{r}(0)=\bar{r}(0)+\Delta(0) - 4\epsilon_{h}-\alpha\bar{r}(0)\geq \Delta(0)+(1-\alpha)\bar{r}(0)- 4\epsilon_{h} >0.
\end{align*}
Therefore, once $h\geq 1+ \log_2(4/(\Delta(0)+(1-\alpha)\bar{r}(0))),$ all the remaining actions in $\cA_h\setminus\{0\}$ would have expected reward that is larger than $\alpha\bar{r}(0)$ and if we solve the safe optimal design \eqref{eq:tabular_spe}, it would set $\pi_h(0)=0.$ 

\subsection{Proof of \cref{lemma:spe7}}\label{sec:lemma:spe7}
Note that
\begin{align*}
\sum_{h=1}^{h_0}\Delta(0)\left(\sum_{a\in\cA_*\cup\{0\}}N_h(a)\right)
=& \Delta(0)	\sum_{h=1}^{h_0}0.5\left(\sum_{a\in\cA_*\cup\{0\}}\pi_h(a)\right)g_h(\pi_h)\epsilon_h^{-2}\log(KT^4/\delta)\\
\leq&0.5\Delta(0)\log(KT^4/\delta)	\sum_{h=1}^{h_0}g_h(\pi_h)\epsilon_h^{-2}\\
\leq&0.5\max_h\left(g_h(\pi_h)\epsilon^{-1}_h\right)\Delta(0)\log(KT^4/\delta)	\sum_{h=1}^{h_0}\epsilon_h^{-1},
\end{align*}
where we use the fact that $\sum_{a\in\cA_*}\pi_h(a)\leq 1$ in the second step. Note that
\begin{align}
    \sum_{h=1}^{h_0}\epsilon_h^{-1}=\sum_{h=1}^{h_0}2^h\leq 2^{h_0+1}=\frac{16}{\Delta(0)+(1-\alpha)\bar{r}(0)},
\end{align}
we have
\begin{align*}
\sum_{h=1}^{h_0}\Delta(0)\left(\sum_{a\in\cA_*\cup\{0\}}N_h(a)\right)
\leq&0.5\max_h\left(g_h(\pi_h)\epsilon^{-1}_h\right)\Delta(0)\log(KT^4/\delta)	\frac{16}{\Delta(0)+(1-\alpha)\bar{r}(0)}\\
\leq&8\max_h\left(g_h(\pi_h)\epsilon^{-1}_h\right)\log(KT^4/\delta).
\end{align*}

\subsection{Proof of \cref{lemma:spe4}}\label{sec:lemma:spe4}
In fact, a straightforward analysis could show that if there exists some action $a\in\cA_h,$ such that $L_h(a)\geq\bar{r}(0),$ the safe optimal design \eqref{eq:tabular_spe_safety} would not allocate any probability mass to the default action 0; Otherwise, it could always reduce $g_h(\pi_h)$ by splitting $\pi_h(0)$ to $\pi_h(a)$ and the action with least probability of being selected. As a result, $\pi_h(0)=0$ if $L_h(a)\geq\bar{r}(0),$ and the conclusion holds automatically.

Now if $L_h(a)<\bar{r}(0)$ for all $a\in\cA_h,$ we let $\underline{L}_{h}=\min_a L_h(a)$ and $\overline{U}_h=\max_aU_h(a)$ and consider the following safe optimal design problem
\begin{align}
		\min \ &g_h(\pi'_h) \nonumber \\
		\mathrm{s.t.} \
		& \pi'_h\in\Delta_{|\cA_h|-1}\,, \nonumber \\ 
		&\min_{r(a)\in[\underline{L}_{h},\overline{U}_{h}]~\forall a}\sum_{a\in\cA_h\setminus\{0\}}\pi'_h(a)r(a)+\pi'_h(0)\bar{r}(0)
		\geq\alpha\bar{r}(0) \,.\label{eq:tabular_spe_safety_conservative}
	\end{align} 
	This is a more conservative version of \eqref{eq:tabular_spe_safety} and it is evident that
	\begin{align}
	    \pi'_h(0)\geq\pi_h(0),\qquad g_h(\pi'_h)\geq g_h(\pi_h)
	\end{align}
	Note that the minimum of the last constraint in \eqref{eq:tabular_spe_safety_conservative} is always achieved at $r(a)=\underline{L}_{h},$ we can thus equivalently re-write this as
		\begin{align}
		\min \ &g_h(\pi'_h) \nonumber \\
		\mathrm{s.t.} \
		& \pi'_h\in\Delta_{|\cA_h|-1}\,, \nonumber \\ 
		&(1-\pi'_h(0))\underline{L}_{h}+(\pi'_h(0)-\alpha)\bar{r}(0)
		\geq0 \,.
		\label{eq:tabular_spe}
	\end{align}
It is thus easy to verify that (note that we assume $h<1+\log_2(4/(\Delta(0)+(1-\alpha)\bar{r}(0)))$)
\begin{align*}
	\pi'_h(0)=1-\frac{(1-\alpha)\bar{r}(0)}{4\epsilon_{h-1}-\Delta(0)},\quad g_h(\pi'_h)=\frac{|\cA_h|-1}{1-\pi'_h(0)}\leq\frac{K(4\epsilon_{h-1}-\Delta(0))}{(1-\alpha)\bar{r}(0)}<\frac{4K\epsilon_{h-1}}{(1-\alpha)\bar{r}(0)}=\frac{8K\epsilon_{h}}{(1-\alpha)\bar{r}(0)}.
\end{align*}
As a result,
\begin{align}
    \pi_h(0)g_h(\pi_h)\leq\pi'_h(0)g_h(\pi'_h)\leq g_h(\pi'_h)\leq\frac{8K\epsilon_{h}}{(1-\alpha)\bar{r}(0)}.
\end{align}

\section{Auxiliary Results}\label{sec:auxiliary}
\begin{lemma}\label{lemma:lagrangian}
For any ellipsoid $\Theta_{\eps}$ defined as
$$\Theta_\eps:
  = \set{\theta \in \realset^d: (\theta - \bar{\theta})\T
  \bar{\Sigma}^{-1} (\theta - \bar{\theta}) \leq \eps}$$
and a vector $b\in\realset^d$ we have
\begin{align}
\max_{\theta\in\Theta_{\eps}}b\T\theta=b\T\bar{\theta}+\sqrt{\eps b\T\bar\Sigma b}
\end{align}
\begin{proof}
By Lagrangian multiplier method, we have
\begin{align}\label{eq:lagrangian01}
    \max_{\theta\in\Theta_{\eps}}b\T\theta=\min_{\lambda\geq 0}\max_{\theta} b\T\theta-\lambda\left[(\theta - \bar{\theta})\T
  \bar{\Sigma}^{-1} (\theta - \bar{\theta})-\eps\right].
\end{align}
Taking the derivative w.r.t. $\theta$ for $b\T\theta-\lambda\left[(\theta - \bar{\theta})\T
  \bar{\Sigma}^{-1} (\theta - \bar{\theta})-\eps\right]$ and setting it to 0, we have
  \begin{align}\label{eq:lagrangian02}
      \theta=\bar{\theta}+\frac{\bar{\Sigma} b}{2\lambda}.
  \end{align}
 Since the quantity $b\T\theta-\lambda\left[(\theta - \bar{\theta})\T
  \bar{\Sigma}^{-1} (\theta - \bar{\theta})-\eps\right]$ is concave in $\theta,$ its maximum is attained when $\theta$ is set to the value as the R.H.S. of \eqref{eq:lagrangian02}. We thus have
 \begin{align}
    \max_{\theta\in\Theta_{\eps}}b\T\theta=\min_{\lambda\geq 0} b\T\bar{\theta}+\frac{b\T\Sigma b}{2\lambda}-\frac{b\T\bar{\Sigma} b}{4\lambda} +\lambda\eps=\min_{\lambda\geq 0} b\T\bar{\theta}+\frac{b\T\bar{\Sigma} b}{4\lambda} +\lambda\eps= b\T\bar{\theta}+\sqrt{\eps b\T\bar{\Sigma} b},
\end{align}
where we use the AM-GM inequality in the last step.
\end{proof}
\end{lemma}

\end{appendices}

\end{document}